%% file: main.tex
\documentclass{article} % For LaTeX2e
\usepackage[svgnames]{xcolor}
\PassOptionsToPackage{pdfborder={0 0 0}, colorlinks=true, linkcolor=DarkGreen, citecolor=DarkBlue}{hyperref}
\usepackage{iclr2026/iclr2026_conference}
\usepackage{times}
\usepackage{amsmath, amsfonts, amssymb}
\usepackage{amsthm}
\usepackage{thmtools}
\usepackage{xspace}
\usepackage{graphicx}
\usepackage{subcaption}
\usepackage{pifont}
\usepackage{listings}
\usepackage{booktabs}
\usepackage[ruled,vlined]{algorithm2e}

\iclrfinalcopy

\input{notation}

\input{theorems}

\usepackage{hyperref}
\usepackage{url}
\usepackage{cleveref}

\title{Generative Bayesian optimization:\\generative models as acquisition functions} %TODO: Open to suggestions on title

\author{Rafael Oliveira, Daniel M. Steinberg \& Edwin V. Bonilla \\
CSIRO's Data61, Australia\\
\scriptsize\texttt{\{rafael.dossantosdeoliveira, dan.steinberg, edwin.bonilla\}@data61.csiro.au} \\
}

\begin{document}

\maketitle

\begin{abstract}
We present a general strategy for turning generative models into candidate solution samplers for batch Bayesian optimization (BO). The use of generative models for BO enables large batch scaling as generative sampling, optimization of non-continuous design spaces, and high-dimensional and combinatorial design. Inspired by the success of direct preference optimization (DPO), we show that one can train a generative model with noisy, simple utility values directly computed from observations to then form proposal distributions whose densities are proportional to the expected utility, i.e., BO's acquisition function values. Furthermore, this approach is generalizable beyond preference-based feedback to general types of reward signals and loss functions. This perspective avoids the construction of surrogate (regression or classification) models, common in  previous methods that have used generative models for black-box optimization. Theoretically, we show that the generative models within the BO process follow a sequence of  distributions which asymptotically approximate an optimal target under certain conditions. We also evaluate the performance through experiments on challenging optimization problems involving large batches in high dimensions. %TODO: Revise!
\end{abstract}

\section{Introduction}
% Many real-world decision-making problems can be formulated as an optimization over a black-box objective function. Often potential actions are costly, observations are noisy, and we can only make a limited number of attempts. In these scenarios, 
Bayesian optimization (BO) has been a successful approach to solve complex black-box optimization problems by making use of probabilistic surrogate models, such as a Gaussian processes (GPs) \citep{Rasmussen2006}, and their uncertainty estimates \citep{Shahriari2016, Garnett2023book}. BO methods have been particularly useful in areas such as hyper-parameter tuning for machine learning algorithms \citep{Snoek2012}, material design \citep{Frazier2016}, and robot locomotion \citep{Calandra2016}. The core idea of BO is to apply a Bayesian decision-theoretic framework to make optimal choices by maximizing an expected utility criterion, also known as an acquisition function. The corresponding expectations are taken under a Bayesian posterior over the underlying objective function. Thus, the Bayesian model provides a principled way to account for the uncertainty inherent to the limited amount of data and the noisy observations. 
% Maybe unnecesary here? move to related work 
% There have been a variety of BO algorithms in the literature extending the basic algorithmic framework to a variety of challenging settings, including high-dimensional \citep{Kandasamy2015}, combinatorial and cost-constrained problems, and the use of non-GP surrogates \citep{Li2024survey}. %TODO: Add more references

In many applications such as simulated scenarios \citep{Azimi2010}, one is able to run multiple evaluations of the objective function in parallel, even though the simulations themselves might be expensive to run. % BO algorithms have been specialized for these applications in a variety of ways, but 
Common BO approaches to these batch settings incrementally build a set of candidates by sampling ``fantasy'' observations from the probabilistic model and conditioning on them before selecting the next candidate in the batch \citep{Wilson2018}. Although near-optimal batches can be selected this way, this approach is not scalable to very large batches in high-dimensional spaces, such as problems in protein design \citep{stanton2022accelerating, Gruver2023}. 

One of the most promising alternatives to batch BO has been to train a generative model as a proposal distribution informed by the acquisition function and then sample a batch from the learned proposal \citep{Brookes2019cbas, stanton2022accelerating, Gruver2023, Steinberg2025variational}. This approach comes with several advantages. Firstly, given a trained generative model, sampling is usually inexpensive. %, so that large batches of samples can be drawn at ease. 
Secondly, % in many applications, generative models have already been developed for other uses, and a BO-like process can be used to fine-tune such models for the optimization task. 
existing general-purpose generative models can be used and fine-tuned for the optimization task at hand. Lastly, sampling avoids estimating the global optimum of an acquisition function, which can be hard. 
% required by BO, which is often highly non-convex presenting multiple local optima, replacing optimization by sampling according to a density that concentrates in regions of high utility. 
% In spite of these advantages, 
However, % replacing point-based optimization by a sampling process also comes with a variety of challenges. Existing methods mostly rely on fitting a surrogate regression or classification model first to obtain expected utility values and then training a generative model on top of it \citep{stanton2022accelerating, Gruver2023, Steinberg2025variational}. 
existing generative approaches to black-box optimization usually rely on fitting a surrogate (regression or classification) model first then training a generative model on top of it \citep{stanton2022accelerating, Gruver2023, Steinberg2025variational}. 
This two-stage process compounds approximation errors from both models and %basically doubles the computational costs when compared to the original BO setting involving a single model. Hence, an ideal generative framework for BO should be able to solely rely on a generative model, while still accounting for uncertainties.
can increase the computational cost significantly when compared to having a single model. 

In this paper we present a general framework for learning generative models for batch Bayesian optimization tasks that requires a single model without the need for additional probabilistic regression or classification surrogates. Our approach for generative BO (GenBO) encodes general utility functions into training objectives for generative models directly. We focus on two cases, one where we train the model via a loss function for a reward model analogously to the direct preference optimization (DPO) formulation for large language models \citep{Rafailov2023dpo}, and the second one where we train the generative model through divergence minimization, using utilities as part of sample weights. We present theoretical analyses on the convergence of approximations and empirical results on practical applications involving high-dimensional combinatorial optimization problems.

\section{Background}
\label{sec:background}

We consider the problem of estimating the global optimum of an objective function $\objective:\domain\to\R$ as:
\begin{equation}
    \location^*\in\argmax_{\location\in\domain} \objective(\location)\,,
    \label{eq:problem}
\end{equation}
where $\objective$ is an expensive-to-evaluate black-box function, i.e., $\nabla_\location \objective$ is unavailable. We can only observe $\objective(\location)$ via noisy evaluations $\observation = \objective(\location) + \obsnoise$, where $\obsnoise$ is often assumed to be zero-mean Gaussian noise, though more general noise models can be handled by the framework we propose.
We assume that the objective $\objective$ can be evaluated in parallel, and the algorithm is allowed to run up to $\niter \geq 1$ optimization rounds with a batch of $\nbatch$ query locations $\batch_\iteridx := \{\location_{\iteridx,i}\}_{i=1}^\nbatch \subset \domain$ per round.

% Besides sequential decisions where a single $\location_\iteridx$ is chosen per round, batch versions of BO take advantage of parallel function evaluations by selecting a batch of query points $\batch_\iteridx := \{\location_{\iteridx,i}\}_{i=1}^\nbatch \subset \domain$ per round. However, to simplify our presentation and avoid notation clutter, we will consider the single-point sequential setting to present the background and the initial derivation of our framework, later extending it to the batch setting.

\paragraph{BO with regression models.} Typically BO assumes a Bayesian prior over $\objective$ \citep{Garnett2023book}, often given by a Gaussian process \citep{Rasmussen2006}. Given a set of observations $\dataset_\iteridx$, corrupted by Gaussian noise $\obsnoise\sim\normal(0,\sigma_\obsnoise^2)$, the Bayesian posterior distribution over $\objective$ given the data $\dataset_\iteridx$ is available in closed form as a GP with known mean and covariance functions (see \autoref{sec:gp}).
% with mean and covariance function given by:
% \begin{align}
%     \gpmean_\iteridx(\location) &:= \vec\kernel_\iteridx(\location)^\transpose (\Kernel_\iteridx + \sigma_\obsnoise^2\eye)^{-1} \observations_\iteridx \label{eq:gp-post-mean}\\
%     \kernel_\iteridx(\location,\location') &:= \kernel(\location,\location') - \vec\kernel_\iteridx(\location)^\transpose (\Kernel_\iteridx + \sigma_\obsnoise^2\eye)^{-1}\vec\kernel_\iteridx(\location')\label{eq:gp-post-cov}
%     % \sigma_\iteridx^2(\location) &:= \kernel_\iteridx(\location,\location),
% \end{align}
% where $\vec\kernel_\iteridx(\location) := [\kernel(\location,\location_i)]_{i=1}^\iteridx \in \R^\iteridx$, $\Kernel_\iteridx := [\kernel(\location_i,\location_j)]_{i,j=1}^\iteridx \in \R^{\iteridx\times\iteridx}$, $\observations_\iteridx := [\observation_i]_{i=1}^\iteridx \in \R^\iteridx$, for $\location,\location'\in\domain$.
BO then uses the model's posterior distribution to compute an acquisition function $\af_\iteridx(\location)$ mapping candidate points $\location\in\domain$ to their expected utility value $\expectation[\utility(\observation)|\location, \dataset_\iteridx]$, where the utility function $\utility$ intuitively encodes how useful it is to collect a new observation at $\location$. Classical examples of expected utilities include the probability of improvement $\af_\iteridx(\location) := p(\observation\geq\thresh | \location, \dataset_\iteridx) = \expectation[\indicator[\observation\geq \thresh]|\location, \dataset_\iteridx]$ and the expected improvement $\af_\iteridx(\location) := \expectation[\max\{\observation-\thresh, 0\}|\dataset_\iteridx]$. The next candidate is then chosen as:
\begin{equation}
    \location_{\iteridx+1} \in \argmax_{\location\in\domain} \af_\iteridx(\location)\,.
\end{equation}

\paragraph{Batch BO.} This strategy can be extended to the batch setting in a variety of ways \citep[\S 11.3]{Garnett2023book}. For instance, one can select the first batch point $\location_{\iteridx,1}$ by maximizing $\af_\iteridx$ as above, and then select the next candidate as $\location_{\iteridx,2} \in \argmax_{\location\in\domain} \expectation[\utility(\observation) | \location, \dataset_\iteridx \cup \{\location_{\iteridx,1}, \tilde\observation_{\iteridx, 1}\}]$, where the expectation is over both $\tilde\observation_{\iteridx,1} \sim p(\observation|\location_{\iteridx,1}, \dataset_\iteridx)$ and $\observation \sim p(\observation|\location, \dataset_\iteridx \cup \{\location_{\iteridx,1}, \tilde\observation_{\iteridx, 1}\})$, and iterate over this process until $\nbatch$ candidates have been selected for parallel evaluation. Although near optimal, evaluating this conditional expectation becomes quickly intractable as the batch size grows. Hence, one usually resorts to Monte Carlo approximations \citep{Wilson2018}. Other BO strategies allow for efficient optimization of the batch in parallel, such as information-theoretic acquisition functions \citep{Takeno2020,Teufel2024}, or even asynchronously \citep{Kandasamy2018}. However, scaling up to large batches in high-dimensional domains, especially involving combinatorial or mixed discrete-continuous search spaces, remains challenging \citep{Gonzalez-Duque2024}.

\paragraph{Active generation with classification models.}
Instead of relying on a Bayesian surrogate model for $\objective$ and then computing an acquisition function $\af$ on top of it, one can model $\af$ directly, which is the main idea behind likelihood-free BO \citep{Song2022lfbo}.
% On this line, \citet{Tiao2021} proposed learning a probabilistic classifier $\classifier(\location)$ to approximate the probability of improvement\footnote{The original proposal in \citet{Tiao2021} was to approximate the expected improvement via a relative density ratio. However, that was later shown to only hold for the probability of improvement in general by \citet{Song2022lfbo}, who then proposed a general framework for other model utilities.} $p(\observation\geq\thresh|\location)$, where $\thresh$ follows a quantile of the empirical data distribution. This framework was later extended to the batch setting by using the learned classifier as a pseudo-likelihood model $p(\observation\geq \thresh|\location)$ to draw samples from $p(\location|\observation\geq \thresh)$ \citep{Oliveira2022batch}. 
% In contrast to approaches based on latent spaces, %TODO: Add citations
On this line, methods like variational search distributions (VSD, \citealp{Steinberg2025variational}) and batch BORE \citep{Oliveira2022batch}
learn a probabilistic classifier $\classifier(\location) \approx p(\observation\geq \thresh)$ in the original space, $\domain$, based on improvement labels $\clabel := \indicator[\observation\geq\thresh]$ and then generate batches by approximately sampling $\nbatch$ candidates from $p(\location|\observation\geq\thresh)$.
% The generative model can then be used to generate batch candidates by sampling $\nbatch$ candidates directly $\batch_\iteridx \overset{\iid}{\sim} \proposal_\iteridx(\location)$ at each round $\iteridx$.
% Namely, let $\clabel := \indicator[\observation \geq \thresh]$, where $\thresh$ is, e.g., a (empirical) quantile of the marginal distribution of observations. We can then learn a class probability estimator $\classifier(\location) \approx p(\clabel=1 | \location) = p(\observation\geq\thresh|\location)$ by minimizing a classification loss, such as the cross-entropy:
The classifier can be learned by, e.g., minimizing the cross-entropy loss:
\begin{equation}
    \Loss_\nobs(\classifier) := -\sum_{i=1}^{\nobs} \clabel_i \log \classifier(\location_i) + (1 - \clabel_i) \log (1 - \classifier(\location_i))\,.
    \label{eq:ce-loss}
\end{equation}
Given a prior $\prior$ over $\domain$ and the classifier $\classifier_\iteridx$ that minimizes $\Loss_{\nobs_\iteridx}$ over the current $\nobs_\iteridx := \iteridx\nbatch$ data points in $\dataset_\iteridx$, we can now learn a generative model approximating $p(\location|\observation\geq \thresh, \dataset_{\iteridx})$ as:
\begin{equation}
    \proposal_\iteridx \in \argmax_{\proposal} \expectation_{\location\sim\proposal}[\log \classifier_\iteridx(\location)] - \kl{\proposal}{\prior}\,,
\end{equation}
which corresponds to an evidence lower bound treating $\classifier_\iteridx(\location) \approx p(\observation \geq \thresh|\location, \dataset_\iteridx)$ as a likelihood.

\paragraph{Direct preference optimization.} The process above for learning $\proposal_\iteridx$ can be likened to the typical fine-tuning of large language models (LLMs) via reinforcement learning with human feedback (RLHF, \citealp{Bai2022rlhf}), which would normally involve training the LLM as an RL agent with a reward model $\reward$. In practice, we do not directly observe rewards, but have access to user preferences. Given a prompt's context $\context$, corresponding to the RL state, let $\location^+, \location^- \sim \proposal(\location|\context)$ denote two answers generated by an LLM $\proposal$, with $\location^+$ denoting the answer preferred by the user, and $\location^-$ the dispreferred one. Having a dataset $\dataset_\nobs^+ := \{\context_i, \location_i^+, \location_i^-\}_{i=1}^\nobs$, one can then learn a reward function $\reward$ by minimizing the negative log-likelihood under a preference model, such as \citet{Bradley1952rank}:
\begin{equation}
    \Loss_\nobs^+(\reward) := - \expectation_{(\context, \location^+, \location^-) \sim \dataset_\nobs^+} [\log\sigma(\reward(\context, \location^+) - \reward(\context, \location^-))]\,.
    \label{eq:rlhf-reward}
\end{equation}
Having learned a reward model $\reward_\nobs$, RLHF trains the LLM as to approximate an agent's optimal policy under $\reward_\nobs$. Regularization based on the Kullback-Leibler (KL) divergence with respect to a reference model $\refmodel$ is further added to improve stability. The optimal generative model then solves:
\begin{equation}
    \proposal_\nobs \in \argmax_{\proposal} \expectation_{\context\sim\dataset_\nobs^+, \location\sim\proposal(\location|\context)}[\reward_\nobs(\context,\location)] - \temperature\kl{\proposal}{\refmodel}\,.
    \label{eq:rlhf-policy}
\end{equation}
Direct preference optimization (DPO, \citealp{Rafailov2023dpo}) removes the need for an explicit reward model by viewing the LLM itself through the lens of a reward model. It is not hard to show that, fixing a reward model $\reward$, the optimal solution to \autoref{eq:rlhf-policy} is given by:
\begin{equation}
    \proposal(\location|\context) = \frac{1}{\normfactor_\reward(\context)} \refmodel(\location|\context)\exp\left(\frac{1}{\temperature}\reward(\context,\location)\right),
\end{equation}
where $\normfactor_\reward(\context) := \sum_\location \refmodel(\location|\context) \exp(\temperature^{-1}\reward(\context,\location))$ is the partition function at the given context $\context$. Although it is intractable to evaluate $\normfactor_\reward$ in practice, DPO uses the fact that, in the Bradley-Terry model, the partition function-dependent terms cancel out. Note that the reward model $\reward$ can be expressed in terms of the optimal $\proposal$ as:
\begin{equation}
    \reward(\context,\location) = \temperature\log\left(\frac{\proposal(\location|\context)}{\refmodel(\location|\context)}\right) + \temperature\log\normfactor_\reward(\context)\,.
\end{equation}
Applying the substitution above to the preference-based loss \eqref{eq:rlhf-reward}, we get:
\begin{equation}
    \Loss_{\mathrm{DPO}}(\proposal) = - \expectation_{(\context, \location^+, \location^-) \sim \dataset_\nobs^+} \left[
        \log\sigma\left(
            \temperature\log\left(\frac{\proposal(\location^+|\context)}{\refmodel(\location^+|\context)}\right)
            -
            \temperature\log\left(\frac{\proposal(\location^-|\context)}{\refmodel(\location^-|\context)}\right)
        \right)
    \right]\,,
\end{equation}
which eliminates the partition function $\normfactor_\reward$ terms.
Therefore, we can train the generative model $\proposal$ directly with $\Loss_{\mathrm{DPO}}$ without the need for an intermediate reward model. Such simplification to a single training loop cuts down the need for computational resources, eliminates a source of approximation errors (from learning $\reward$), and brings in theoretical guarantees from Bradley-Terry models \citep{Shah2016, Bong2022}. The main question guiding our work is whether we can apply a similar technique to simplify the training of (arbitrary, not necessarily LLM) generative models for likelihood-free BO by removing the need for an intermediate surrogate model for $\objective$.

\section{A general recipe for generative Bayesian optimization}
As seen in \Cref{sec:background}, using generative models for BO typically involves training a regression or classification model as an intermediate step to then train the candidate generator. The use of an intermediate model demands additional computational resources and brings in further sources of approximation errors which may hinder performance. %TODO: Expand on this or move it to introduction.
Hence, we propose a framework to train the generative model directly from (noisy) observation values. The main idea is to train the model to approximate a target distribution proportional to BO's acquisition function and then use the learned generative model as a proposal for the next query locations. There are different approaches to do so, some of which have been previously explored in the literature for specific acquisition functions, such as the probability of improvement \citep{Brookes2019cbas, Steinberg2025variational} and upper confidence bound \citep{Yun2025dibo}. However, we here focus on a general recipe to turn a generative model into a density following \emph{any} acquisition function that can be expressed as an expected utility.

\paragraph{Utility functions.} Consider a likelihood-free BO setting \citep{Song2022lfbo}, where we aim to directly learn an acquisition function $\af_\iteridx:\domain\to\R$ from available data $\dataset_\iteridx$. If our acquisition function takes the form of an expected utility:
\begin{equation}
    \af_\iteridx(\location) = \expectation[\utility_\iteridx(\observation) \mid \location, \dataset_\iteridx],
\end{equation}
we can estimate $\af_\iteridx$ from noisy utility data $\{\location_i, \utility_{\iteridx,i}\}_{i=1}^{\nobs_\iteridx}$, where $\utility_{\iteridx,i} = \utility_\iteridx(\observation_i)$. As examples of utility functions, we have:
\begin{enumerate}
    \item Probability of improvement (PI): $\utility_\iteridx(\observation) = \indicator[\observation \geq \thresh_\iteridx]$;
    \item Expected improvement (EI): $\utility_\iteridx(\observation) = \max(\observation - \thresh_\iteridx, 0)$;
    \item Soft expected improvement (sEI): $\utility_\iteridx(\observation) = \operatorname{softplus}(\observation-\thresh_\iteridx)$;
    \item Mean: $\utility_\iteridx(\observation) = \observation$;
\end{enumerate}
given a threshold $\thresh_\iteridx$ for improvement-based utilities, e.g., the largest observation value or a quantile of the empirical observations distribution \citep{Tiao2021}. A comprehensive summary of typical utility functions for BO can be found in \citet{Wilson2018}. The ones listed above, however, can be directly expressed as a function of the observations. The soft-plus version of EI (sEI) is added as a smoother utility variant of EI which remains positive when $\observation \leq \thresh$ \citep[cf.][]{Ament2023}.

\paragraph{BO with generative models.} As an illustrative example, consider the case of PI where $\af(\location) = \expectation[\indicator[\observation \geq \thresh]\mid \location] = p(\observation \geq \thresh|\location)$, which has been previously applied to train generative models for black-box optimization using reward surrogate models \citep{Steinberg2025variational}. Given a sampler for the conditional distribution $p(\location | \observation \geq \thresh)$, by Bayes rule, we recover the original PI as:
\begin{equation}
    \af(\location) = p(\observation \geq \thresh|\location) = \frac{p(\location | \observation \geq \thresh)p(\observation \geq \thresh)}{\prior(\location)} \propto \frac{p(\location | \observation \geq \thresh)}{\prior(\location)}\,.
\end{equation}
The prior $\prior$ is usually known, and it can be set as uninformative $\prior(\location) \propto 1$ or set to encode prior information about the optima. We then see that learning a generative model to approximate the posterior above is equivalent to learning a probabilistic classifier for the improvement event $\observation \geq \thresh$, as in \citet{Song2022lfbo}. Moreover, if we only have a probabilistic classifier approximating $p(\observation \geq \thresh|\location)$, we still need to select candidate points via optimization over the classification probabilities landscape, which can be highly non-convex presenting several local optima, as in the usual BO setting we choose $\location_{\iteridx+1}$ as the (global) maximizer of the acquisition function $\af$. In contrast, a generative model provides us with a direct way to sample candidates $\location \sim p(\location | \observation \geq \thresh)$ which will naturally concentrate in the regions of highest probability density, and therefore highest utility, according to the model. Finally, note that this same reasoning can be extended to any other non-negative expected utility function by training the generative model to approximate:
\begin{equation}
    p_\iteridx^*(\location) \propto \prior(\location) \af_\iteridx(\location)\,,
\end{equation}
or similarly $p_\iteridx^*(\location) \propto \prior(\location) \exp \af_\iteridx(\location)$, which allows for utilities that might take negative values. Thus, GenBO admits a generalized Bayesian interpretation as direct inference over the optimum location $\location^*$ rather than $\objective$, with $\prior$ serving as a prior over $\location^*$ and the acquisition value $\af_\iteridx(\location)$ as a utility-based pseudo-likelihood factor \citep[cf.][]{Knoblauch2022}.

\paragraph{Overview.} Let $\qfamily \subset \Pspace(\domain)$ be a learnable family of probability distributions over a domain $\domain$. We consider general loss functions of the form:
\begin{equation}
    \Loss_\iteridx(\proposal) := \regfactor_\iteridx\regfun_\iteridx(\proposal) + \sum_{i=1}^{\nobs_\iteridx} \loss_i(\proposal)\,,
\end{equation}
where $\loss_i$ are individual losses over points $\location_i \in \domain$ or pairs of points $\location_{i,1}, \location_{i,2}\in \domain$ and their corresponding utility values, $\regfactor_\iteridx \geq 0$ is an optional regularization factor, and $\regfun_\iteridx: \qfamily\to[0,\infty)$ is a complexity penalty function. The algorithm starts by learning a proposal from available data $\dataset_{\iteridx-1}$:
\begin{equation}
    \proposal_{\iteridx-1} \in \argmin_{\proposal \in \qfamily} \Loss_{\iteridx-1}(\proposal)\,.
\end{equation}
A batch $\batch_{\iteridx} := \{\location_{\iteridx,i}\}_{i=1}^\nbatch$ is sampled from the learned proposal $\proposal_{\iteridx-1}$. We evaluate the utilities $\utility_{\iteridx}(\observation_{\iteridx,i})$ with the collected observations $\observation_{\iteridx,i} \sim p(\observation|\location_{\iteridx,i})$, for $i\in \{1,\dots, \nbatch\}$, and repeat the cycle up to a given number of iterations $\niter \in \N$. This process is summarized in \autoref{alg:genbo}. In the following, we describe strategies to formulate general losses $\loss_i$ to learn acquisition functions and how to ensure that the sequence of batches $\{\batch_\iteridx\}_{\iteridx=1}^\infty$ asymptotically concentrates at the global optima.

\begin{algorithm}[t]
    \caption{Generative BO}
    \label{alg:genbo}
    \DontPrintSemicolon
    \KwIn{Domain $\domain$, initial data \(\dataset_0\)}
    \For{$\iteridx \in \{1,\dots, \niter\}$}{
        $\proposal_{\iteridx-1} \in \argmin_{\proposal\in\qfamily} \Loss_{\iteridx-1}(\proposal)$   \tcp*{Fit proposal distribution}
        $\batch_\iteridx = \{\location_{\iteridx,i}\}_{i=1}^\nbatch \sim \proposal_{\iteridx-1}$ \tcp*{Sample batch}
        $\observation_{\iteridx, i} \leftarrow \objective(\location_{\iteridx,i}) + \obsnoise_{\iteridx,i}$, for $i \in \{1, \dots, \nbatch\}$ \tcp*{Collect observations}
        $\dataset_\iteridx = \dataset_{\iteridx-1} \cup \{\location_{\iteridx,i}, \observation_{\iteridx,i}\}_{i=1}^\nbatch$  \tcp*{Update data}
    }
\end{algorithm}

\subsection{Preference-based learning}
We first aim to apply a similar reparameterization trick to the one in DPO to simplify generative BO methods. Note that, for a general classification loss, such as the one in \autoref{eq:ce-loss}, it is not possible to eliminate the partition function resulting from a DPO-like reparameterization without resorting to approximations, which might change the learned model. Hence, we need a pairwise-contrastive training objective.

\paragraph{Preference loss.} To apply a preference-based loss, we can train a model to predict preferential directions of the acquisition function. Assume we have a dataset $\udataset_\nobs := \{\location_i, \utility_i\}_{i=1}^\nobs$ with $\nobs$ evaluations of a given utility function $\utility: \R \to \R$.  We may reorganize the data into pairs of inputs and corresponding utility values $\{\location_{i,1}, \location_{i,2}, \utility_{i,1}, \utility_{i,2}\}_{i=1}^{\nobs/2}$, where $\utility_{i,j} := \utility(\observation_{i,j})$, for $j\in\{1,2\}$, and train a generative model $\proposal$ using the Bradley-Terry preference loss from DPO with, for $i\in\{1,\dots, \nobs/2\}$:
\begin{equation}
    \loss^\PL_i(\proposal) := \loss^\PL_i(\proposal, \Delta \utility_i) := -\log \sigma \left( 
        \temperature
        \sign(\Delta\utility_i) 
        \left(
            \log\left(
                \frac{\proposal(\location_{i,1})}{\prior(\location_{i,1})}
            \right)
            -
            \log\left(
                \frac{\proposal(\location_{i,2})}{\prior(\location_{i,2})}
            \right)
        \right)
    \right),
    \label{eq:loss-pl-summand}
\end{equation}
where $\Delta\utility_i := \utility_{i,1} - \utility_{i,2}$, as in the DPO formulation, $\temperature > 0$ is a (optional) temperature parameter and the prior $\prior$ can be given by a reference model, either pre-trained or derived from expert knowledge about feasible solutions to the optimization problem \eqref{eq:problem}. Similar to \citet{Rafailov2023dpo}, the learned generative model is seeking to approximate:
\begin{equation}
    p_\utility^*(\location) := \frac{1}{\normfactor_\utility} \prior(\location)\exp\left(\frac{1}{\temperature}\expectation[\utility(\observation)|\location]\right),
\end{equation}
where $\normfactor_\utility$ is the normalization factor.

\paragraph{Robust preference loss.} As shown in \citet{Chowdhury2024rdpo}, the original DPO loss is not robust to preference noise. As in BO, one usually only observes noisy evaluations of the objective function, utility values directly derived from the observation values will also be noisy and correspondingly the sign of their differences as well. Namely, assume there is a small $\pflip \in (0, 1/2)$ probability of the preference directions being flipped w.r.t. the sign of the true expected utility:
\begin{equation}
    \prob{\sign(\utility_{i,1} - \utility_{i,2}) = \sign(\expectation[\utility_{i,2}|\location_{i,2}] - \expectation[\utility_{i,1}|\location_{i,1}])} = \pflip\,.
\end{equation}
\citet{Chowdhury2024rdpo} showed that the original DPO preference loss is biased in this noisy case, and proposed a robust version of the DPO loss to address this issue as:
\begin{equation}
    \loss_i^\rPL(\proposal) := \frac{(1-\pflip)\loss_\PL(\proposal, \Delta\utility_i) - \pflip\loss_i^\PL(\proposal, -\Delta\utility_i)}{1-2\pflip},
    \label{eq:loss-rpl}
\end{equation}
which yields the robust preference loss (rPL). %: $\Loss_\nobs^\rPL(\proposal) := \sum_{i=1}^\nobs \loss_i^\rPL(\proposal)$. 
It follows that the loss function above is unbiased and robust to observation noise under mild assumptions \citep{Chowdhury2024rdpo}.

% Given a pair of inputs $\location_1, \location_2 \in \domain$, let $\location^+$ denote the input $\location\in\{\location_1, \location_2\}$ with the highest value of $\objective(\location)$ and $\location^-$ the one with the lowest. In practice, we might not have access to $\objective(\location)$, but only noisy observations $\observation = \objective(\location) + \obsnoise$, which we can use to order the input pairs by comparing the corresponding observations $\observation_1$ and $\observation_2$. In fact, with any form of partial ordering for the (implicit) outcomes of the input pairs $(\location_1, \location_2)$, we can obtain pairs of $(\location^+, \location^-)$ to use within a preference-based loss like the Bradley-Terry model \eqref{eq:rlhf-reward}, making it applicable to preferential BO problems \citep{Gonzalez2017preferential}.

\subsection{Divergence-based learning}
A disadvantage of DPO-based losses when applied to BO is that they only take the signs of the pairwise utility differences into account, discarding the remaining information contained in the magnitude of the utilities. A simpler approach is to train the generative model $\proposal$ to match $p_\utility^*$ directly.

\paragraph{Forward KL.} If we formulate the target distribution as $p_\utility^*(\location) \propto \prior(\location) \af(\location)$, the forward Kullback-Leibler (KL) divergence of the proposal w.r.t. the target is given by:
\begin{equation}
    \kl{p_\utility^*}{\proposal} = \expectation_{\location\sim p_\utility^*}[\log p_\utility^*(\location) - \log \proposal(\location)]\,.
\end{equation}
As we do not have samples from $p_\utility^*$, at each iteration $\iteridx$ the algorithm generates samples from the current best approximation $\batch_\iteridx := \{\location_{\iteridx,i}\}_{i=1}^\nbatch \sim \proposal_{\iteridx-1}$. An unbiased training objective can then be formulated as:
\begin{equation}
    \loss_i^{\mathrm{fKL}}(\proposal) = -\frac{\prior(\location_i)}{\proposal_{i-1}(\location_i)}\utility(\observation_i)\log\proposal(\location_i)\,,
    \label{eq:loss-fkl}
\end{equation}
which we write in a condensed form to avoid notation clutter where $\proposal_{i-1}$ denotes the proposal density used to sample $\location_i$ under the re-indexed sequence of observations. The objective above is unbiased and its global optimum can be shown to converge to $p_\utility^*$ by an application of standard results from the adaptive importance sampling literature \citep{Delyon2018ais}. A simpler version of this training objective was derived for CbAS \citep{Brookes2019cbas} using only the last batch for training, which would allow for convergence only as the batch size goes to infinity $\nbatch\to\infty$. Furthermore, as we discuss in our analysis, convergence to $p_\utility^*$ is not sufficient to ensure convergence to the global optima of the objective function $\objective$.

\paragraph{Balanced forward KL.} As utilities like those of PI and EI can evaluate to 0 at the points where $\observation<\thresh$ was observed, with $\thresh$ corresponding to an improvement threshold, every point below the threshold will not be penalized by the loss function. As a result, the model may keep high probability densities in regions of low utility. To prevent this, we may use an alternative formulation of the forward KL which comes from the definition of Bregman divergences with the convex function $\utility\mapsto\utility\log\utility$, yielding a loss:
\begin{equation}
    \loss_i^{\mathrm{bfKL}}(\proposal) = -\frac{\prior(\location_i)}{\proposal_{i-1}(\location_i)}\utility(\observation_i)\log\proposal(\location_i) + \frac{\proposal(\location_i)}{\proposal_{i-1}(\location_i)}\,.
    \label{eq:loss-bfkl}
\end{equation}
We defer the details of the derivation to the appendix. Although the additional $\proposal(\location)$ only contributes to a constant term when integrated over, for finite-sample approximations, it contributes to a soft penalty on points where it is observed that $\utility(\observation) = 0$.

\subsection{Generalizations}
In general, we can extend the above framework to use a proper scoring rule $\score: \Pspace(\domain) \times \domain \to \R$ \citep{Gneiting2007strictly} other than the log loss. We can then optimize $\proposal$ using losses of the form:
\begin{equation}
    \loss_i^\score(\proposal) = - \frac{\prior(\location_i)}{\proposal_{i-1}(\location_i)} \utility(\observation_i)\score(\proposal, \location_i)\,.
\end{equation}
Although we leave the exploration of this formulation for future work, it is essentially compatible with our proposed theoretical framework and readily extensible to generative models that may not have densities available in closed form, such as diffusion and flow matching \citep{Lipman2024}, which still provide flexible probabilistic models.

\section{Theoretical analysis}
In this section, we present a theoretical analysis for the algorithm's approximation of the utility-based target distribution $p_\utility^*$ and its performance in regards to the global optimization problem \eqref{eq:problem}.
%We will start with the presentation of general results for the learning of parametric generative models $\proposal_\parameters(\location)$ in the dependent-data settings of BO, noting that most of the related literature has focused on results only applicable to non-parametric settings with kernel methods and linear models \citep{Srinivas2010, Takeno2024, Chowdhury2024rdpo}. We will instead use the framework of kernel methods as a tool for analysis, without requiring a kernel model. We will then move on to the analysis of the concentration of the target distribution around the global optimum of the objective function.
We consider parametric generative models $\proposal_\parameters$ with a given parameter space $\paramspace\subset\R^\paramdim$. For the purpose of our analysis, we will assume that models can be described as $\proposal_\parameters(\location) = \exp\model_\parameters(\location)$, which is possible whenever probability densities are strictly positive $\proposal_\parameters(\location) > 0$. To accommodate for both the pairwise preference-based losses and the point-based divergence approximations, we introduce the following notation for the loss function:
\begin{equation}
    \Loss_\nobs(\model_\parameters) = \regfun_\nobs(\model_\parameters) + \sum_{i=1}^\nobs \weight_i\loss(\observer_i(\model_\parameters), \mlabel_i)\,,
\end{equation}
where $\observer_i(\model_\parameters)$ corresponds to the $i$th model evaluation with, e.g., $\observer_i(\model_\parameters) := \model_\parameters(\location_i) = \log \proposal_\parameters(\location_i)$ for KL, and $\observer_i(\model_\parameters) := \model_\parameters(\location_{i,1}) - \model_\parameters(\location_{i,2})$ for preference-based losses, $\mlabel_i$ encodes the dependence on utility values with $\mlabel_i := \utility(\observation_i)$ for KL and $\mlabel_i := \sign(\utility_{i,1} - \utility_{i,2})$ for DPO losses, and $\weight_i$ are potential importance weights.
% We set $\bar\regfun_\nobs$ as an extended regularizer $\bar\regfun_\nobs(\model) := \regfactor_\nobs\regfun_\nobs(\model) + \frac{\bar\regfactor_\nobs}{2}(\int_\domain \exp\model(\location) \diff\basemeasure(\location) - 1)^2$, where $\basemeasure$ corresponds to the underlying base measure on the domain $\domain$ (i.e., the counting measure for discrete domains or the Lebesgue measure for Euclidean spaces). Note that the additional term is always zero for the generative models, as $\int_\domain \exp\model_\parameters(\location) \diff\basemeasure(\location) = \int_\domain \proposal_\parameters(\location)\diff\basemeasure(\location) = 1$, but including it here facilitates our analysis to operate with any unconstrained $\model:\domain\to\R$.

\paragraph{Regularity assumptions.} We make a few mild regularity assumptions about the problem setting and the model. Firstly, for the analysis, we assume that the models $\model_\parameters$ lie in a reproducing kernel Hilbert space (RKHS) $\Hspace_\kernel$ shared with the true log density $\model_*$, which is such that $p_\utility^*(\location) = \exp \model_*(\location)$. This assumption does not impose that the proposals are kernel models, but simply that they lie in Hilbert space of functions with well defined point evaluations. The domain $\domain$ is assumed to be a compact metric space, with main results specialized for the finite discrete setting, i.e., $\card{\domain} < \infty$. 
% The model $\proposal_\parameters(\location)$ is continuously twice differentiable with respect to the parameters $\parameters\in\paramspace$ with bounded second-order derivatives. 
The individual losses $\loss: \R \times \R \to \R$ are strictly convex and twice differentiable w.r.t. their first argument. 
% In addition, we assume that, at the target $\model_*$ the individual loss $\loss(\observer_i(\model_*), \mlabel_i)$ is conditionally sub-Gaussian \citep{Boucheron2013, Chowdhury2017} w.r.t. the data-generating process, basically meaning that the probability distribution of each loss has zero mean and light tails. 
We also assume that the regularizer $\regfun_\nobs$ is strongly convex and twice differentiable. The rest of our assumptions and proofs are presented and discussed in more detail in \autoref{sec:analysis}.

\begin{restatable}{lemma}{lossbounds}
    \label{thr:loss-bounds}
	Let assumptions \ref{a:regularization} to \ref{a:weights} be satisfied. Then,
	\begin{equation*}
		\frac{1}{2} \norm{\model - \model_\nobs}_{\Hessian_\nobs}^2 \leq\Loss_\nobs(\model) - \Loss_\nobs(\model_\nobs) \leq \frac{1}{2} \norm{\nabla\Loss_\nobs(\model)}_{\Hessian_\nobs^{-1}}^2\,,
		% \label{eq:loss-bounds}
	\end{equation*}
	where $\Hessian_\nobs:\Hspace_\kernel\to\Hspace_\kernel$ is an operator-valued lower bound on the Hessian of the loss $\Loss_\nobs$:
	\begin{equation*}
		\forall\model\in\Hspace_\kernel, \quad \nabla^2 \Loss_\nobs(\model) \succeq \Hessian_\nobs := \regfactor_\nobs\idop + \lossfactor\bound_\weight\sum_{i=1}^\nobs \observer_i \otimes \observer_i\,.
	\end{equation*}
\end{restatable}
\begin{remark}
    The result in \autoref{thr:loss-bounds} automatically ensures that the loss functional $\Loss_\nobs$ is strongly convex over the function space, as $\nabla^2 \Loss_\nobs(\model) \succeq \Hessian_\nobs \succeq \regfactor_\nobs\idop \succ \operator{0}$, for all $\model\in\Hspace_\kernel$, and therefore has a unique minimizer at $\model_\nobs$, defined in \autoref{sec:assumptions}. Note, however, that the same cannot be implied about $\Loss_\nobs(\model_\parameters)$ over $\paramspace$ based on this result alone, since the mapping $\parameters\mapsto\model(\cdot, \parameters)$ may not be linear.
\end{remark}

% \begin{restatable}{corollary}{errorbound}
%     \label{thr:error-bound-simple}
%     Consider the setting in \autoref{thr:loss-bounds}, and assume that there is $\parameters_* \in \paramspace$ such that $\model_{\parameters_*} = \model_*$. Then, given any $\delta\in(0,1)$, the following holds with probability at least $1-\delta$:
%     \begin{equation*}
%         \forall\nobs\in\N, \quad \lvert \inner{\observer, \model_*}_\kernel - \inner{\observer, \model_{\parameters_\nobs}}_\kernel \rvert 
%         \leq 
%         2\beta_\nobs(\delta)\norm{\observer}_{\Hessian_\nobs^{-1}},
%         \quad \forall \observer\in\Hspace_\kernel,
%     \end{equation*}
%     where $\beta_\nobs(\delta) := \regfactor^{-1/2}\norm{\nabla\bar\regfun_\nobs(\model_*)}_\kernel + \sigma_\loss \sqrt{2\lossfactor^{-1}\log(\det(\eye + \lossfactor\regfactor^{-1}\observers_\nobs^\transpose\observers_\nobs)^{1/2}/\delta)}$, and $\observers_\nobs := [\observer_1, \dots, \observer_\nobs]$.
% \end{restatable}

\begin{restatable}{theorem}{mainthm}
    \label{thr:error-bound}
    Let assumptions \ref{a:rkhs} to \ref{a:dense} hold. Then, given any $\delta\in(0,1)$,
    \begin{equation*}
        \forall\nobs\in\N, \quad \lvert \inner{\observer, \model_*}_\kernel - \inner{\observer, \model_{\parameters_\nobs}}_\kernel \rvert 
        \leq 
        \beta_\nobs(\delta) (\norm{\observer}_{\Hessian_\nobs^{-1}} + 2|\inner{\observer,\unitf}_\kernel| v_\nobs(\bar\proposal_\nobs) ),
        \quad \forall \observer\in\Hspace_\kernel,
    \end{equation*}
    % where $\beta_\nobs(\delta) := \regfactor^{-1/2}\norm{\nabla\bar\regfun_\nobs(\model_*)}_\kernel + \sigma_\loss \sqrt{2\lossfactor^{-1}\log(\det(\eye + \lossfactor\regfactor^{-1}\observers_\nobs^\transpose\observers_\nobs)^{1/2}/\delta)}$, and $\observers_\nobs := [\observer_1, \dots, \observer_\nobs]$.
    which holds with probability at least $1-\delta$, where $\beta_\nobs(\delta)$ is given by \autoref{thr:error-unconstrained}, 
    where $\unitf(\location) = 1$ denotes the unit constant function, $v_\nobs(\proposal) := \Ex[\location\sim\proposal]{\norm{\feature(\location)}_{\Hessian_\nobs^{-1}}}$, and $\bar\proposal_\nobs := \frac{p_\utility^* + \proposal_\nobs}{2}$.
    % $\condnum(\Hessian_\nobs)$ is the condition number of $\Hessian_\nobs$, and $\anyconstant_\kernel > 0$ is constant that only depends on the kernel $\kernel$.
\end{restatable}

The result above shows that the approximation error for the optimal parameter $\parameters_\nobs$ concentrates similarly to that of a kernel method, even though we do not require the model to be a kernel machine. In addition, the term $\norm{\observer}_{\Hessian_\nobs^{-1}}$ is associated with the predictive variance of a Gaussian process model, which can be shown to converge to zero if, e.g., $\inf_{\location\in\domain}\proposal_\parameters(\location) \geq \bound_\proposal > 0$, for all $\parameters \in \paramspace$ (see \autoref{thr:gp-variance-convergence} in the appendix). Alternatively, one can follow a KL-focused analysis as in \citet{Oliveira2021} to bound $\kl{\proposal_{\parameters_\nobs}}{\proposal_*}$ via the information gain, which bounds the growth of predictive uncertainty terms $\norm{\observer_\nobs}_{\Hessian_{\nobs-1}^{-1}}^2$ appearing in \autoref{thr:error-bound}. The asymptotic rate of $\beta_\nobs$ is discussed in \autoref{thr:vanishing-error}.

\paragraph{Optimality.} \autoref{thr:error-bound} allows us to establish that the model converges to the target $\model_*$ associated with the target distribution $p_\utility^*$ for a given utility function $\utility$, provided $\norm{\observer}_{\Hessian_\nobs^{-1}}$ vanishes sufficiently fast. However, convergence to the target distribution alone does not ensure optimality of the samples $\location \sim \proposal_\iteridx$. The latter is possible by applying results from reward-weighted regression, which shows that training a proposal to maximize $\expectation_{\observation\sim p(\observation|\location), \location\sim \proposal_{\iteridx-1}}[\utility(\observation)\log \proposal(\location)]$ yields a sequence of increasing expected rewards $\expectation[\utility(\observation_\iteridx)] \leq \expectation[\utility(\observation_{\iteridx+1})] \leq \dots$ \citep[Thm. 4.1]{Strupl2022rwr}. If the maximizer of the sequence of expected utilities converges to the maximizer of the objective function $\objective$, then the generative BO proposals will concentrate at $\objective$'s global optima. Therefore, for KL-based loss functions, one may drop the proposal densities in the importance sampling weights $1/\proposal_{i-1}(\location_i)$ to promote this posterior concentration phenomenon, as corroborated by our experimental findings, which generally did not include importance weights. This same concentration of the learned target distribution should also occur with the preference-based loss functions due to the absence of importance-sampling weights. If these targets concentrate on the optimizer set, simple regret may vanish, whereas sublinear cumulative regret requires rate control that we leave for future work.
% We leave a rigorous theoretical regret analysis for future work.

\section{Related work}

Using generative models for online-optimization is becoming an increasingly popular strategy for optimization in discrete, mixed discrete-continuous or high-dimensional design spaces where classical
BO is limited. The following discusses other works applying generative models to BO settings and contrasts them with the reward-model-free active generation framework we propose.

\paragraph{Latent-space BO.} In latent-space BO (LSBO) methods for high-dimensional problems \citep{Gomez-Bombarelli2018, stanton2022accelerating, Gruver2023}, one learns a probabilistic representation of a (usually lower-dimensional) manifold of the data jointly with $\objective$, and performs BO in that space, projecting query points back to the original space at evaluation time. This technique has led to numerous BO methods for high-dimensional and discrete-space optimization \citep{Gomez-Bombarelli2018, Gruver2023, Gonzalez-Duque2024}. Learning this latent space can, however, cause complications. LSBO can suffer poor sample efficiency if the latent space is learned from the initial training set and then fixed \citep{tripp2020sample}. Or poor performance can arise from reconstruction errors between the latent and observation space \citep{lee2025latent}. GenBO and other methods like VSD do not suffer from these issues as all inference is done in the observation space. Lastly, despite recent advances in the field \citep{Chu2024inversionbased, lee2025latent, Moss2025cowboys}, to our knowledge, LaMBO-2 remains state-of-the-art in LSBO for \emph{long} sequences, like proteins.

\paragraph{Diffusion for BBO.} There has been recent progress in adapting denoising diffusion models to black-box optimization (BBO) tasks, often by learning a model that can be conditioned on observation values, given a dataset of evaluations \citep{Krishnamoorthy23a}. Other approaches involve guiding the diffusion process by a given utility function derived from a regression model \citep{Gruver2023, Yun2025dibo}. Note, however, that such methodologies are specific to diffusion, whereas we focus on a general approach that can be applied to arbitrary generative models.
%TODO: Discuss preferential and Meta-BO papers where they learn a generative model based on BO data

\paragraph{LLMs and BO.} Recent work has begun to integrate large language models (LLMs) into BO pipelines, primarily to inject prior knowledge, improve cold-start performance, or offload certain design decisions to a learned policy. Several studies use LLMs as contextual priors over the design space: for example, guiding initialization or proposal generation by leveraging natural-language domain knowledge \citep{liu2024large}, or selecting acquisition functions adaptively via an LLM-driven controller \citep{aglietti2025funbo}. Other work treats BO as a test-time search tool that an LLM can call to refine or validate its own proposals during inference \citep{agarwal2025searching}. Most relevant to our setting is a recent reward-model-free approach for protein engineering \citep{chen2025generalists}, which uses LLM preference modeling, akin to DPO, to steer search without an explicit surrogate. This shares the reward-model-free philosophy of GenBO, but differs fundamentally in relying on a general-purpose LLM, whereas GenBO provides a framework for task-specific generative black-box optimization problems with no language interface or pretrained reward structure.

\section{Experiments}
\label{sec:experiments}
We evaluate several variants of generative BO (GenBO) on a number of challenging sequence optimization tasks against popular and strong baselines, including CbAS \citep{Brookes2019cbas}, VSD \citep{Steinberg2025variational}, and LaMBO-2 \citep{Gruver2023}, besides trivial baselines, random mutations and a genetic algorithm (GA) implemented in \textsc{poli} \citep{Gonzalez-Duque2024}. As performance measures, we assess the simple regret, $\regret_\iteridx := \objective(\location^*) - \max_{i\leq\nobs_\iteridx} \objective(\location_i)$, and the cumulative maximum, $\max_{i\leq \nobs_\iteridx} \objective(\location_i)$, where $\nobs_\iteridx := \card{\dataset_\iteridx}$ is the number of function evaluations up to round $\iteridx$. In legend boxes, algorithms are sorted in descending order of final average regret. Shaded areas correspond to $\pm 1$ standard deviation across five different random seeds. \autoref{app:experiment-settings} presents further details about experiment settings and ablation studies. \autoref{tab:results} and \ref{tab:results-foldx} summarize final results.

\subsection{Text optimization} % Ngram problem

As a first experiment, we wish to optimize a short sequence (5 letters) to minimize the edit distance to the sequence \texttt{ALOHA}, which is implemented as a \textsc{poli} black-box \citep{Gonzalez-Duque2024poli}. Here $\domain = \vocab^M$ where $\vocab$ is the English alphabet, and $M$ is sequence length. Even though this sequence is relatively short, still $|\domain| = |\vocab|^M > 11.8$ million elements. We increase the difficulty by only allowing $|\dataset_0| = 64$ where the minimum edit distance is 4, $\nbatch = 8$, and $T=10$. We compare GenBO to the classifier guided VSD \citep{Steinberg2025variational} and CbAS \citep{Brookes2019cbas}, and to a simple greedy baseline that applies (3) random mutations to its best candidates per-round \citep{Gonzalez-Duque2024poli}. For GenBO, VSD and CbAS we use a simple mean-field (independent) categorical proposal distribution, $\proposal$, and a uniform prior, $\prior$. VSD and CbAS use a simple embedding and 1-hidden layer MLP classifier for estimating PI. We also varied the threshold $\thresh$ annealing schedule. Architectural details and other experimental specifics are given in \autoref{asub:textgen}.

Results are summarized in \autoref{fig:textopt}. We can see that the random baseline is not able to make much headway and CbAS under-performs due to its limited use of data (last batch only) in retraining. %\todo. 
For this experiment, GenBO with the robust preference loss (rPL) and EI-based utilities showed the quickest improvements, whereas PI is able to reach the exact optimum at the end, with VSD eventually also achieving good performance. In \autoref{fig:annealing} (appendix), we present an ablation study on the threshold $\thresh_\iteridx$ annealing scheme we used to balance the exploration-exploitation trade-off for GenBO and PI-based baselines (VSD and CbAS). The plots reveal that this problem generally favors a more exploitative approach by concentrating on higher quantiles of the observations marginal distribution. GenBO was, however, relatively less sensitive to the choice of annealing scheme, as long as the final percentile was set anywhere above 90\%, whereas VSD required a generally sharper rise to above the 95\% quantile towards the end of the optimization process, favoring original settings suggested by \citet{Steinberg2025variational}. We also find that in this problem the use of a pre-trained informative prior $\prior$ may not bring significant performance advantage, as GenBO variants with no prior (i.e., $\prior\propto 1$) performed best. Lastly, we also highlight significant improvements in run time for GenBO, making it on average 3 times faster than VSD (see \autoref{tab:runtime} in the appendix) for not needing to fit an intermediate surrogate model.

\begin{figure}[htb]
    \centering
    \subcaptionbox{\texttt{ALOHA}\label{fig:textopt}}{
        \includegraphics[width=0.31\linewidth]{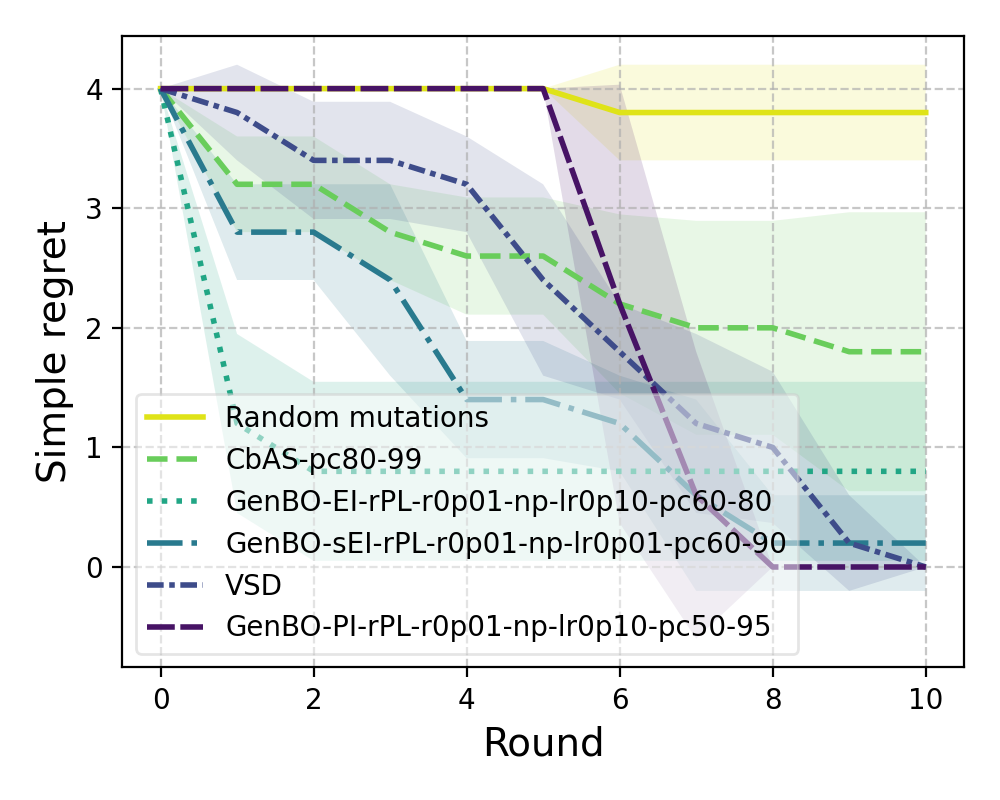}
    }
    \subcaptionbox{Stability\label{fig:stability}}{
        \includegraphics[width=0.31\linewidth]{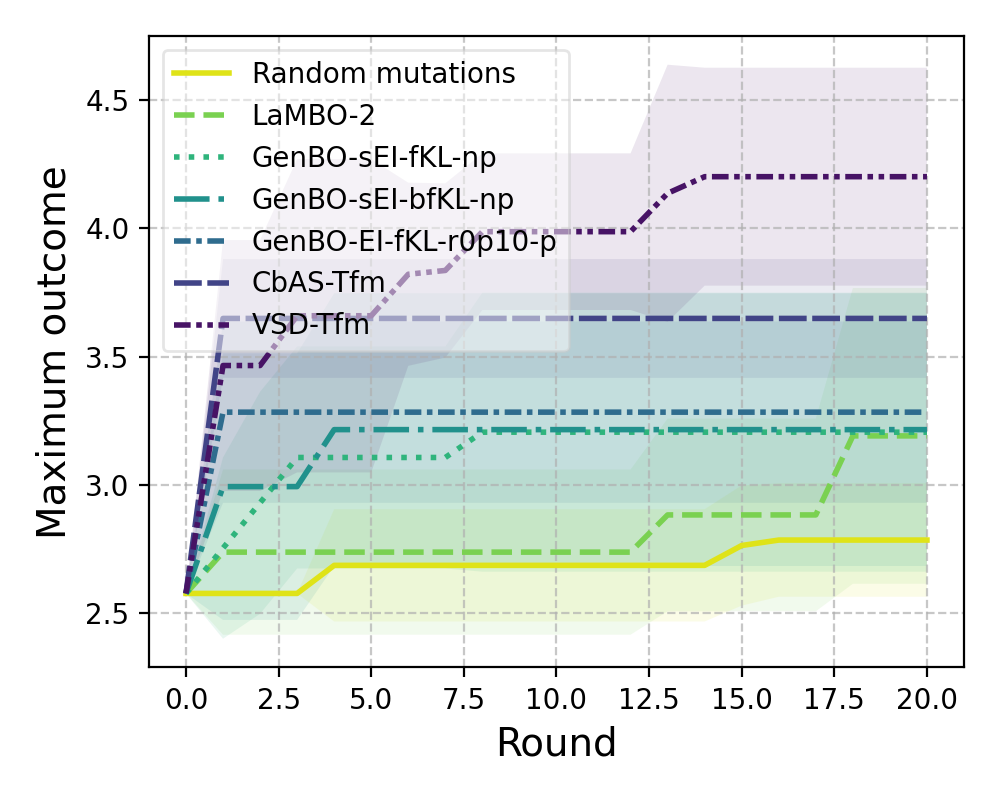}
    }
    \subcaptionbox{SASA\label{fig:sasa}}{
        \includegraphics[width=0.31\linewidth]{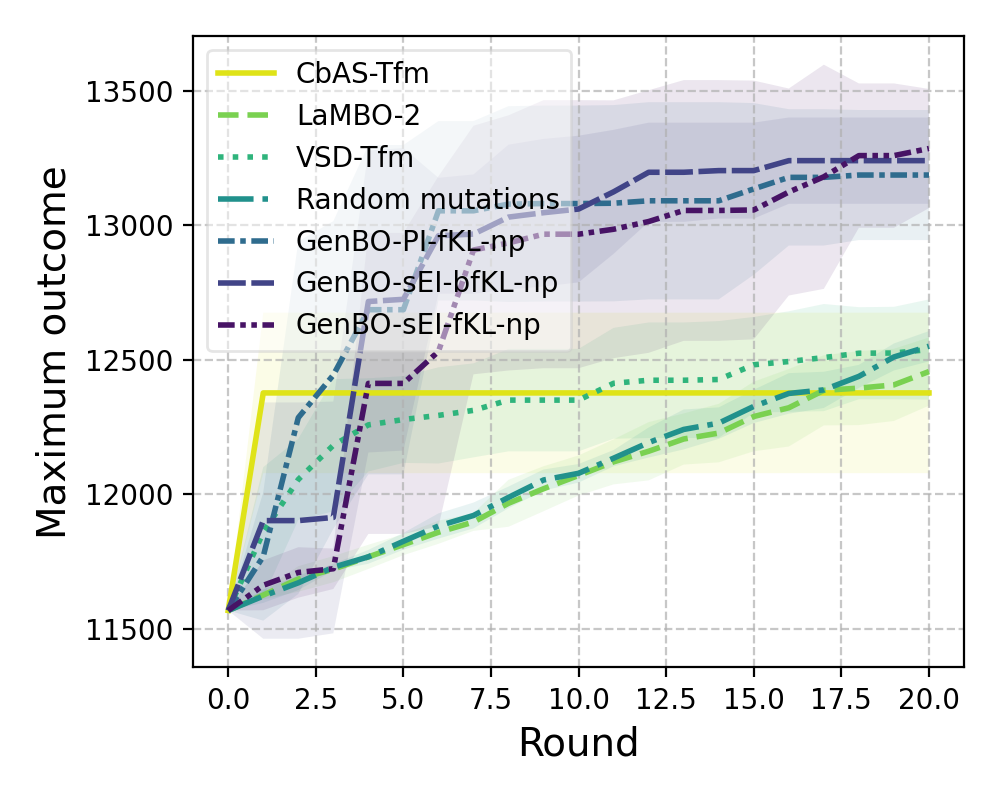}
    }
    \caption{Performance of baseline black box optimizers and GenBO variants on the (a) %(\protect\subref{fig:textopt})
    \texttt{ALOHA}, (b) stability, and (c) solvent accessible surface area optimization problems.}
    \label{fig:textoptfoldx}
\end{figure}

\begin{figure}[htb]
    \centering
    \subcaptionbox{$M=15$\label{fig:em15}}{
        \includegraphics[width=0.315\linewidth]{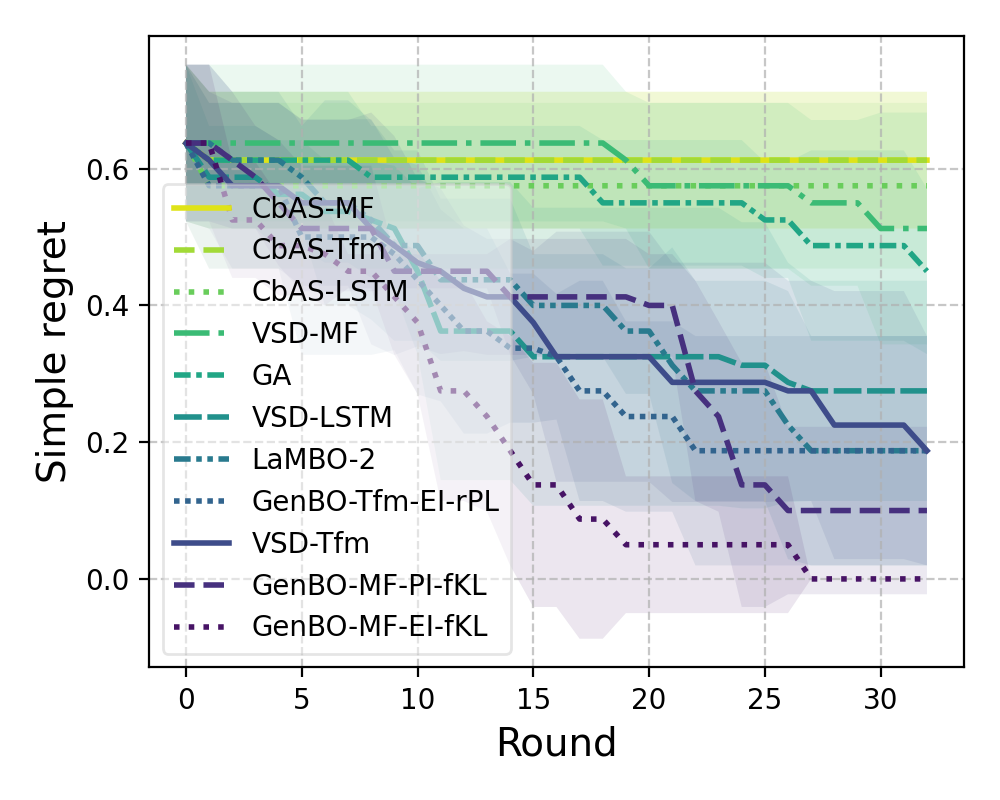}
    }
    \subcaptionbox{$M=32$\label{fig:em32}}{
        \includegraphics[width=0.315\linewidth]{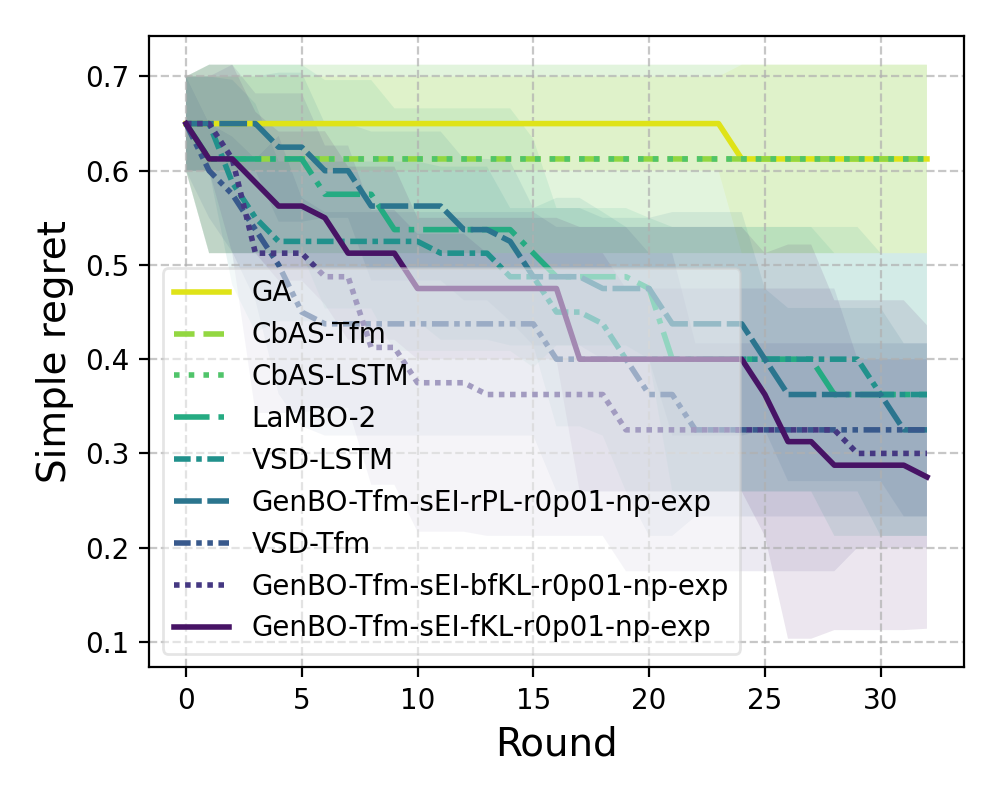}
    }
    \subcaptionbox{$M=64$\label{fig:em64}}{
        \includegraphics[width=0.315\linewidth]{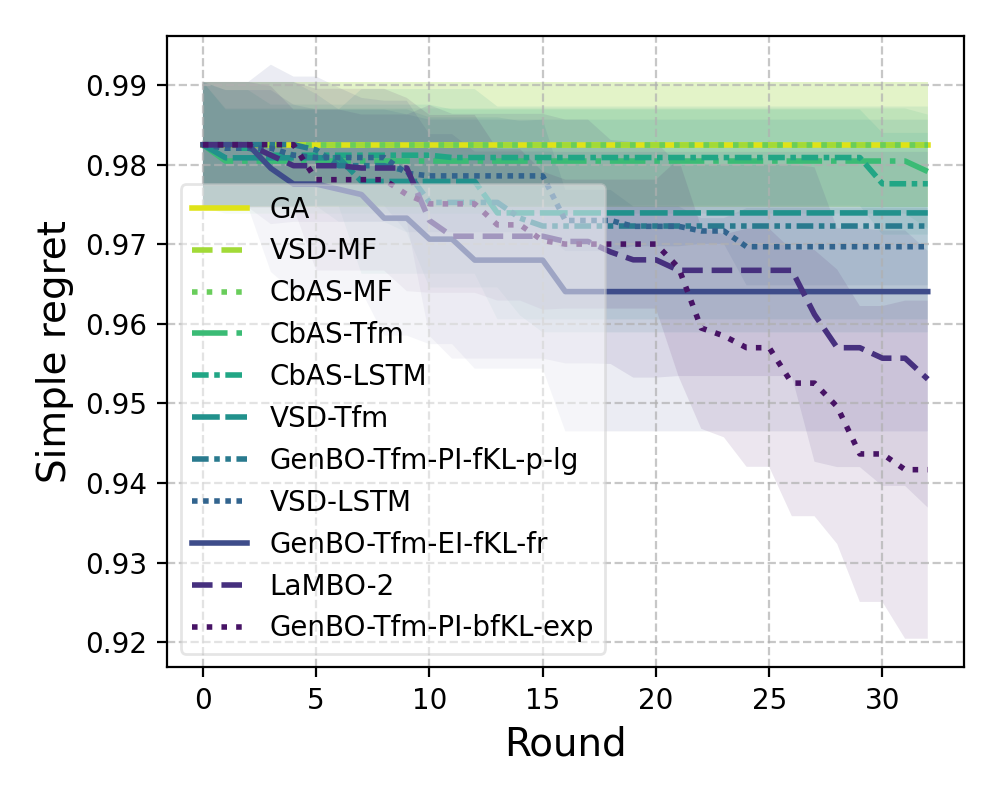}
    }
    \caption{Simple regret of the baseline black box optimizers and the GenBO variants on the Ehrlich closed-form test function protein design task for varying sequence lengths, $M$.}
    \label{fig:ehrlich}
\end{figure}

\subsection{Protein design}

We now consider three protein sequence design tasks where $|\vocab| = 20$ and we have varying $M$. We again use VSD, CbAS and random mutation as baselines, and add to them the guided diffusion based LaMBO-2~\citep{Gruver2023}. GenBO, VSD and CbAS all share the same generative backbone, which is the causal transformer used in \citet{Steinberg2025variational}; VSD and CbAS also use the same CNN-classifier guide used in that work. We present additional architectural information, and additional experimental details in \autoref{asub:protdesign}. We use the black-box implementations in \textsc{poli} for these tasks, and \textsc{poli-baselines} implementations of the random and LaMBO-2 baselines.

The first task we consider is optimization of the Ehrlich functions introduced by \citet{stanton2024closed}. These are challenging biologically inspired parametric closed-form  functions that explicitly simulate nonlinear (epistatic) effects of sequence on outcome. The outcomes are $\observation \in \{-1\} \cup [0, 1]$ where $-1$ is reserved for infeasible sequences. We use the same protocol as in \citet{Steinberg2025variational}, where we optimize sequences of length $M=\{15,32,64\}$ all with motif lengths of 4, and $|\dataset_0|=128$, $T=32$ and $\nbatch = 128$.
The results are summarized in \autoref{fig:ehrlich}. We again see that GenBO variants are able to outperform or match the performance of baselines, with KL-based losses yielding the best performance. In higher dimensions with the longest sequence setting, the benefits of the balanced forward KL loss, with its density minimization effect in areas of lower utility, are more evident. In addition, we note that exponential regularization, corresponding to assuming an exponential dot-product kernel for the RKHS feature space of the model (see \autoref{thr:model-norm}),  allowed for the best performance in higher dimensions. Lastly, in \autoref{fig:batchsize-ablation} (appendix), we present an ablation study on the batch size setting $\nbatch$, showing monotonic improvements, especially for large $\nbatch \geq 32$.

\begin{figure}[htb]
    \centering
    % \subcaptionbox{\texttt{ALOHA}}{\includegraphics[width=0.32\linewidth]{fig/all-diversity-aloha-ablation.png}}
    \subcaptionbox{Stability\label{fig:diversity-stability}}{\includegraphics[width=0.315\linewidth]{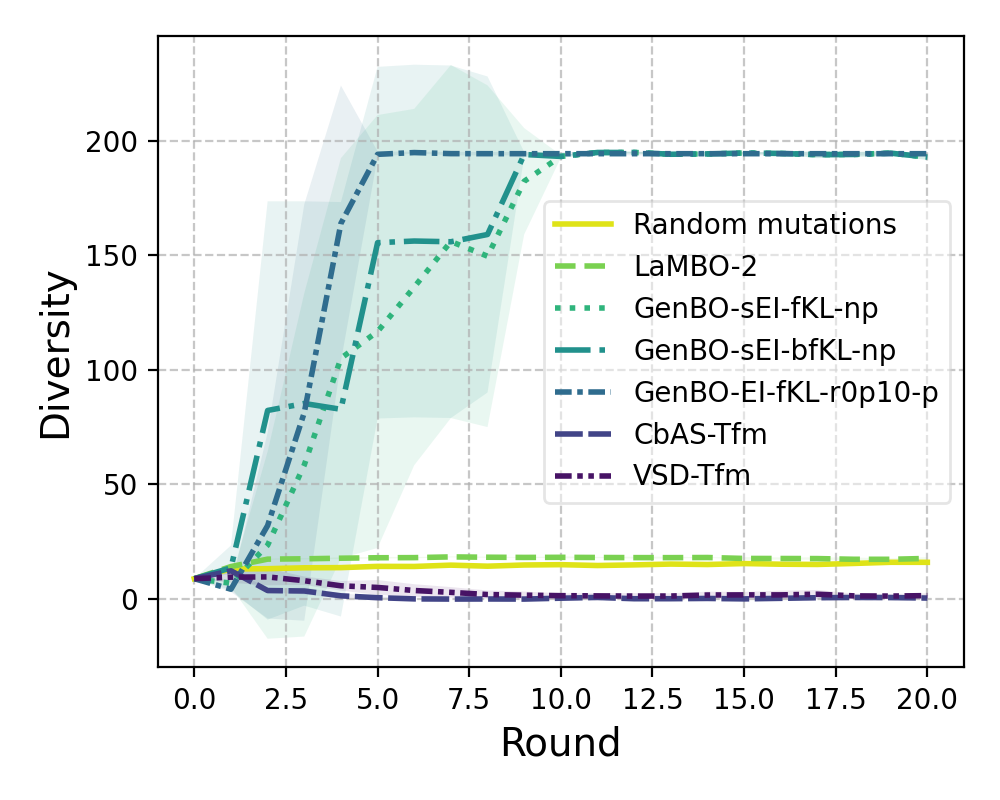}}
    \subcaptionbox{SASA\label{fig:diversity-sasa}}{\includegraphics[width=0.315\linewidth]{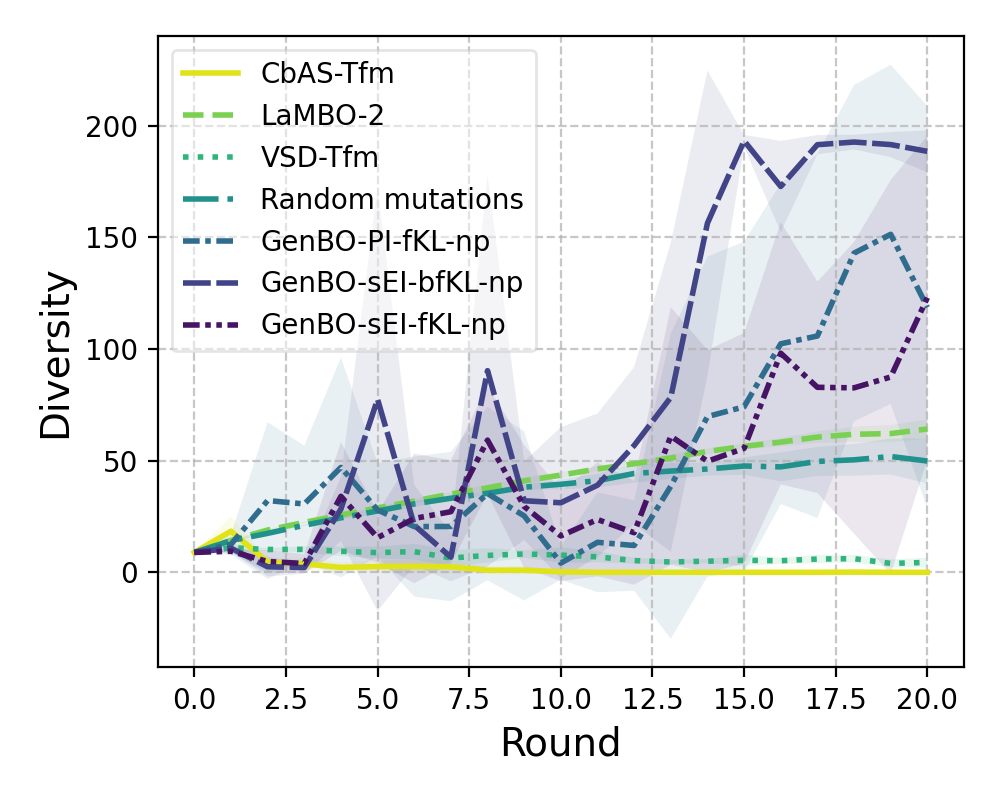}}
    \caption{Batch diversity scores per round on the FoldX protein optimization tasks.}
    \label{fig:diversity}
\end{figure}

For our final set of experiments we present two real protein optimization tasks. These experiments have been adapted from \citet{stanton2022accelerating} where the aims are to maximize the stability and solvent accessible surface area (SASA) of the proteins, respectively. The black-box is the FoldX molecular simulation software \citep{Schymkowitz2005foldx}, and is wrapped by \textsc{poli} \citep{Gonzalez-Duque2024poli}.
We chose the \texttt{mRouge} red fluorescent protein ($M = 228$) as the base protein for the tasks. Both tasks were run for $T=20$ rounds with a batch size of $B=64$ and an initial training set of $|\dataset_0| = 88$ as a subset from \citet{stanton2022accelerating}. Results are summarized in \autoref{fig:stability} for stability and \autoref{fig:sasa} for SASA. All variants of GenBO find the stability task challenging, along with the LaMBO-2 and random baselines. CbAS and especially VSD are better able to stabilize this protein. %, we believe this is because \todo. 
As shown by diversity scores in \autoref{fig:diversity-stability}, which we measure by averaging the Levenshtein distance across the batch in the same way as \citet{Steinberg2025variational}, algorithmic baselines with the lowest diversity yielded top performance, indicating that pure exploitation from around the starting dataset led to the highest outcomes.
However, most variants of GenBO far outperform the baselines on the SASA task, and much more rapidly. We believe this task favors extrapolation away from the prior, due to the high performance of GenBO variants with uninformative prior. In contrast to the stability, the diversity scores show that increasing exploration led to better outcomes for SASA (\autoref{fig:diversity-sasa}). %Final results are summarized in \autoref{tab:results} and \ref{tab:results-foldx}. % \todo. 

\section{Conclusion}

This work introduces Generative Bayesian Optimization (GenBO), a unifying framework that turns any generative model into a sampler whose density tracks BO acquisition functions. We have shown that loss functions over generative models, such as DPO and KL divergences, can be applied to directly learn samplers for batch BO. By eliminating intermediate regression or classification surrogates, GenBO reduces approximation error, simplifies the pipeline to learning just a single generative model, and scales naturally to large batches and high-dimensional or combinatorial design spaces. Theoretical results show convergence to the target distribution, and experiments on text optimization and protein design tasks demonstrate competitive performance with more complex surrogate-guided baselines.
A few challenges remain. For some variants, GenBO requires choosing and fixing the prior before optimization, and its performance depends on sensible settings of utility and temperature parameters, whose theory could be further explored. Another avenue is the adaptation to acquisition strategies not expressible as expected utilities, such as Thompson sampling and upper confidence bound. Despite these caveats, GenBO’s minimal moving parts and principled acquisition-driven training mark a simpler and more scalable alternative to multi-stage guided generation methods.

% ===================================================================================
% NOTE: After acceptance, we may include the following sections.
% ===================================================================================
% \subsubsection*{Author Contributions}
% If you'd like to, you may include  a section for author contributions as is done
% in many journals. This is optional and at the discretion of the authors.

% \subsubsection*{Acknowledgments}
% Use unnumbered third level headings for the acknowledgments. All
% acknowledgments, including those to funding agencies, go at the end of the paper.

\bibliographystyle{iclr2026/iclr2026_conference}
\bibliography{references}

\newpage

\appendix
% \section*{Appendix}

\section{Gaussian processes for BO}
\label{sec:gp}
Assume a Gaussian process prior over $\objective$, e.g., $\objective \sim \gp(0, \kernel)$, where $\kernel:\domain\times\domain\to\R$ is a positive-definite kernel \citep{Rasmussen2006}. Then, given a set of observations $\dataset_\nobs := \{\location_i, \observation_i\}_{i=1}^\nobs$, corrupted by Gaussian noise $\obsnoise\sim\normal(0,\sigma_\obsnoise^2)$, the posterior $\objective|\dataset_\nobs \sim \gp(\gpmean_\nobs, \kernel_\nobs)$ is available in closed form with mean and covariance function given by:
\begin{align}
    \gpmean_\nobs(\location) &:= \vec\kernel_\nobs(\location)^\transpose (\Kernel_\nobs + \sigma_\obsnoise^2\eye)^{-1} \observations_\nobs \label{eq:gp-post-mean}\\
    \kernel_\nobs(\location,\location') &:= \kernel(\location,\location') - \vec\kernel_\nobs(\location)^\transpose (\Kernel_\nobs + \sigma_\obsnoise^2\eye)^{-1}\vec\kernel_\nobs(\location')\label{eq:gp-post-cov}\\
    \sigma_\nobs^2(\location) &:= \kernel_\nobs(\location,\location), \label{eq:gp-post-var}
\end{align}
where $\vec\kernel_\nobs(\location) := [\kernel(\location,\location_i)]_{i=1}^\nobs \in \R^\nobs$, $\Kernel_\nobs := [\kernel(\location_i,\location_j)]_{i,j=1}^\nobs \in \R^{\nobs\times\nobs}$, $\observations_\nobs := [\observation_i]_{i=1}^\nobs \in \R^\nobs$, for $\location,\location'\in\domain$. With these closed-form expressions, GP models allow BO algorithms to quantify uncertainty and assess expected utilities of their decisions. However, note that, due to matrix inversions, exact GP inference incurs a computational cost of $\bigo(\nobs^3)$. Hence, one often has to resort to low-rank approximations to make GP predictions tractable in cases involving large amounts of data, such as batch evaluations with large batch size. Alternatively, one may completely discard the GP models and use other surrogates, such as neural networks, and there has been an increasing literature on how to reliably quantify uncertainty for BO when using these models \citep{Li2024}.

% \section{The GenBO algorithm}

% \begin{algorithm}
%     \caption{GenBO}
%     \label{alg:genbo}
%     \DontPrintSemicolon
%     \KwIn{Domain $\domain$, initial data \(\dataset_0\)}
%     \For{$\iteridx \in \{1,\dots, \niter\}$}{
%         $\proposal_\iteridx \in \argmin_{\proposal\in\qfamily} \Loss_{\iteridx-1}(\proposal)$   \tcp*{Fit proposal distribution}
%         $\batch_\iteridx \overset{\iid}{\sim} \proposal_\iteridx$ \tcp*{Sample batch}
%         $\observation_{\iteridx, i} \leftarrow \objective(\location_{\iteridx,i}) + \obsnoise_{\iteridx,i}$, for $i \in \{1, \dots, \nbatch\}$ \tcp*{Collect observations}
%         $\dataset_\iteridx = \dataset_{\iteridx-1} \cup \{\location_{\iteridx,i}, \observation_{\iteridx,i}\}_{i=1}^\nbatch$  \tcp*{Update data}
%     }
% \end{algorithm}

\section{Learning parametric models with RKHS convex losses}
\label{sec:analysis}
In this section, we consider the general problem of learning a function $\model_*$ with a parametric model $\model:\domain\times\paramspace\to\R$, where the parameter space $\paramspace$ is an arbitrary finite-dimensional vector space. Most existing results in the Bayesian optimization and bandits literature for learning these models from inherently dependent data are only valid for linear models or kernel machines. As we will consider arbitrary generative models, we need to derive convergence results applicable to a wider class of models, accommodating popular modern frameworks. %To do so, we will not assume identifiability at first, so that it is not necessary that some $\parameters_* \in \paramspace$ exists such that $\model_* = \model(\cdot; \parameters_*)$. Instead, we will replace identifiability with a much milder assumption that $\model_*$ lies in a reproducing kernel Hilbert space (RKHS) large enough to also contain the models.
To do so, we will assume that there exists a reproducing kernel Hilbert space (RKHS) containing the models and the target log-density $\model_* = \log p_\utility^*$ for a given (fixed) utility function $\utility$.

%%%%%%%%%%%%%%%%%%%%%%%%%%%%%%
\subsection{Main definitions}
%%%%%%%%%%%%%%%%%%%%%%%%%%%%%%

We will need the following definitions to state our main assumptions and results.

\begin{definition}[RKHS]
    \label{def:rkhs}
    Let $\domain$ be a non-empty set. A Hilbert space $\Hspace$ of real-valued functions over $\domain$ is called a reproducing kernel Hilbert space if function evaluations are bounded, that is:
    \begin{equation*}
        \forall \location\in\domain, \quad \exists \anyconstant_\location < \infty: \quad  \forall \anyfunction\in\Hspace, \quad \anyfunction(\location) \leq \anyconstant_\location\norm{\anyfunction}_{\Hspace}.
    \end{equation*}
\end{definition}

Any Hilbert space $\Hspace$ satisfying the condition above is associated with a unique positive-semidefinite kernel $\kernel:\domain\times\domain\to\R$ with the reproducing property \citep{Scholkopf2002}:
\begin{equation}
    \forall\anyfunction\in\Hspace \qquad \anyfunction(\location) = \inner{\anyfunction, \kernel(\cdot, \location)}_{\Hspace}, \qquad \forall\location\in\domain,
\end{equation}
where $\inner{\cdot, \cdot}_{\Hspace}$ denotes the inner product in $\Hspace$, which also defines the norm $\norm{\anyfunction}_{\Hspace} = \sqrt{\inner{\anyfunction, \anyfunction}}_{\Hspace}$, for $\anyfunction\in\Hspace$. In fact, it can be shown that $\Hspace$ is spanned by $\{\kernel(\cdot, \location)\}_{\location\in\domain}$ \citep[Thm. 4.21]{Steinwart2008}. Hence, given a positive-semidefinite kernel $\kernel$, we will use the notation $\inner{\cdot, \cdot}_\kernel$ and $\norm{\cdot}_\kernel$ to denote the inner product and the norm, respectively, in the corresponding RKHS $\Hspace_\kernel$.

\begin{definition}[Strong convexity]
	\label{def:sc}
	A differentiable function $\objective: \Hspace\to\R$ on a Hilbert space $\Hspace$ is $\alpha$-strongly convex over $\Sspace \subseteq \Hspace$, for a given $\alpha > 0$, if:
	\begin{equation*}
		\forall \anyfunction, \anyfunction' \in \Sspace, \quad \objective(\anyfunction) \geq \objective(\anyfunction') + \inner{\nabla\objective(\anyfunction'), \anyfunction - \anyfunction'} + \frac{\alpha}{2} \norm{\anyfunction - \anyfunction'}_\Hspace^2\,.
	\end{equation*}
\end{definition}

\begin{definition}[Smoothness]
	\label{def:smoothness}
	A function $\objective: \Hspace\to\anotherspace$ between Hilbert spaces $\Hspace$ and $\anotherspace$ is $\smoothness$-smooth over $\Sspace \subseteq \Hspace$, for a given $\smoothness > 0$, if:
	\begin{equation}
		\forall \anyfunction, \anyfunction' \in \Sspace, \quad \norm{\objective(\anyfunction) - \objective(\anyfunction')}_\anotherspace \leq \smoothness \norm{\anyfunction - \anyfunction'}_\Hspace\,.
	\end{equation}
\end{definition}

\begin{definition}[Sub-Gaussianity]
	\label{def:sub-g}
	A random variable $\vnoise$ taking values in a Hilbert space $\Hspace$ is said to be $\vncov$-sub-Gaussian, given a positive-definite trace-class operator $\vncov:\Hspace\to\Hspace$, if:
	\begin{equation}
		\forall \anyfunction\in\Hspace, \quad \expectation[\exp\inner{\anyfunction, \vnoise}] \leq \exp \left( \frac{1}{2} \inner{\anyfunction, \vncov \anyfunction} \right).
	\end{equation}
	Likewise, a $\Hspace$-valued stochastic process $\{\vnoise_\nobs\}_{\nobs=1}^\infty$ adapted to the filtration $\{\filtration_\nobs\}_{\nobs=0}^\infty$ is conditionally $\vncov$-sub-Gaussian if the following almost surely holds:
	\begin{equation}
		\forall \anyfunction\in\Hspace, \quad \expectation[\exp\inner{\anyfunction, \vnoise_\nobs} \mid \filtration_{\nobs-1}] \leq \exp \left( \frac{1}{2} \inner{\anyfunction, \vncov \anyfunction} \right), \quad \forall \nobs \in \N.
	\end{equation}
\end{definition}

%%%%%%%%%%%%%%%%%%%%%%%%%%%%%%
\subsection{Auxiliary results}
%%%%%%%%%%%%%%%%%%%%%%%%%%%%%%
\begin{lemma}
	\label{thr:model-rkhs}
	Let $\model: \domain\times\paramspace\to\R$ represent a class of models parameterized by $\parameters\in\paramspace$. Assume that $\model(\location;\cdot) \in \Hspace_\paramspace$, for all 
	$\location\in\domain$, where $\Hspace_\paramspace$ is a reproducing kernel Hilbert space associated with a positive-definite kernel $\kernel_\paramspace:\paramspace\times\paramspace\to\R$. It then follows that:
	\begin{equation}
		\Hspace_\model := \{\anyfunction: \domain \to \R \mid \exists \weight\in\Hspace_\paramspace: \anyfunction(\location) = \inner{\weight, \model(\location, \cdot)}_{\Hspace_\paramspace}, \forall \location\in\domain \}
	\end{equation}
	equipped with the norm:
	\begin{equation}
		\norm{\anyfunction}_{\Hspace_\model} := \inf\{\norm{\weight}_{\Hspace_\paramspace} : \weight \in \Hspace_\paramspace, \anyfunction(\location) = \inner{\weight, \model(\location, \cdot)}_{\Hspace_\paramspace}, \forall\location\in\domain \}
		\label{eq:rkhs-norm}
	\end{equation}
	constitutes the unique RKHS where $\kernel_\model: (\location, \location') \mapsto \inner{\model(\location, \cdot), \model(\location', \cdot)}_{\Hspace_\paramspace}$ is the reproducing kernel.
	% Define the models linear span as:
	% \begin{equation}
		%     \Hspace_0 := \left\lbrace \sum_{i=1}^\nfeatures \anyscalar_i \model(\cdot; \parameters_i) \mmiddle| \nfeatures\in\N, \anyscalar_ \parameters\in\paramspace \right\rbrace
		% \end{equation}
\end{lemma}
\begin{proof}
	This is a direct application of classic RKHS results \citep[e.g.,][Thm. 4.21]{Steinwart2008} where we are treating $\feature:\location\mapsto\model(\location, \cdot)$ as a feature map mapping into an existing Hilbert space $\Hspace_\paramspace$ and taking advantage of its structure to define a new one.
\end{proof}

\begin{remark}
	\label{thr:model-norm}
	The RKHS $\Hspace_\model$ described above has the special property that, for any $\parameters\in\paramspace$, the RKHS norm of the model is given by:
	\begin{equation}
		\norm{\model(\cdot, \parameters)}_{\Hspace_\model}^2 = \kernel_\paramspace(\parameters,\parameters)\,,
	\end{equation}
	since $\inner{\kernel_\paramspace(\cdot, \parameters), \model(\location, \cdot)}_{\Hspace_\paramspace} = \model(\location, \parameters)$ for all $\location\in\domain$, and $\kernel_\paramspace(\cdot, \parameters)$ is the unique representation of the evaluation functional at $\parameters$ in the RKHS $\Hspace_\paramspace$. The rest follows from the definition in \autoref{eq:rkhs-norm}. Hence, each choice of $\kernel_\paramspace$ gives us a potential RKHS norm regularizer.
\end{remark}

\begin{lemma}[{\citealp[Cor. 3.6]{Abbasi-Yadkori2012}}]
	\label{thr:noise-bound}
	Let $\{\filtration_\iteridx\}_{\iteridx=0}^\infty$ be an increasing filtration, $\{\obsnoise_\iteridx\}_{\iteridx=1}^\infty$ be a real-valued stochastic process, and $\{\feature_\iteridx\}_{\iteridx=1}^\infty$ be a stochastic process taking values in a separable real Hilbert space $\Hspace$, with both processes adapted to the filtration. Assume that $\{\feature_\iteridx\}_{\iteridx=1}^\infty$ is also predictable, i.e.,  $\feature_\iteridx$ is $\filtration_{\iteridx-1}$-measurable, and that $\obsnoise_\iteridx$ is conditionally $\sigma_\obsnoise^2$-sub-Gaussian, for all $\iteridx \in \N$. Then, given any $\delta \in (0,1)$, with probability at least $1-\delta$, 
	\begin{equation*}
		\forall \iteridx \in \N, \quad 
		\flexnorm{\sum_{i=1}^\iteridx \obsnoise_i \feature_i}_{
			(\operator{V} + \features_\iteridx\features_\iteridx^\transpose)^{-1}
		}^2 
		\leq 2\sigma_\obsnoise^2 \log \left(
		\frac{
			\det(\eye + \features_\iteridx^\transpose \operator{V}^{-1}\features_\iteridx)^{\frac{1}{2}}
		}{\delta}
		\right),
	\end{equation*}
	for any positive-definite operator $\operator{V} \succ 0$ on $\Hspace$, and where we set $\features_\iteridx := [\feature_1, \dots, \feature_\iteridx]$.
\end{lemma}

\begin{lemma}[{GP variance upper bound {\citep[Lem. E.5]{Steinberg2025variational}}}]
    \label{thr:gp-variance-convergence}
    Let $\{\location_\nobs\}_{\nobs\geq 1}$ be a sequence of $\domain$-valued random variables adapted to the filtration $\{\filtration_\nobs\}_{\nobs\geq 1}$. For a given $\location\in\domain$, assume that the following holds:
    \begin{equation}
    	\exists \niter_* \in \N: \quad \forall\niter \geq \niter_*, \quad \sum_{\nobs=1}^\niter \prob{\location_\nobs = \location \mid \filtration_{\nobs-1}} \geq \bound_\niter > 0\,,
    \end{equation}
    for a some sequence of lower bounds $\{\bound_\nobs\}_{\nobs\in\N}$. Then, for a bounded kernel $\kernel:\domain\times\domain\to\R$ given observations  at $\{\location_i\}_{i=1}^\nobs$, the following holds with probability 1:
	\begin{equation}
	    \sigma_\nobs^2(\location) \in \set{O}(\bound_\nobs^{-1}).
	\end{equation}
	In addition, if $\bound_\nobs \to \infty$, then $\lim_{\nobs\to\infty} \bound_\nobs\sigma_\nobs^2(\location) \leq \sigma_\obsnoise^2$.
\end{lemma}

%%%%%%%%%%%%%%%%%%%%%%%%%%%%%
\subsection{Main assumptions}
\label{sec:assumptions}
%%%%%%%%%%%%%%%%%%%%%%%%%%%%%
\paragraph{Loss.} Given data $\dataset_\nobs := \{\location_i, \observation_i\}_{i=1}^\nobs$, for the analysis, we consider loss functionals defined over an RKHS $\Hspace_\kernel$ with the form:
\begin{equation}
    \Loss_\nobs(\model) = \regfun_\nobs(\model) + \sum_{i=1}^\nobs \weight_i \loss(\observer_i(\model), \mlabel_i)\,, \quad \model \in \Hspace_\kernel,
\end{equation}
where $\loss: \R \times \R \to \R$ is a fixed deterministic function, $\weight_i > 0$ represents importance sampling weights (or $\weight_i \propto 1$ when importance weights are not used), $\observer_i:\Hspace_\kernel\to\R$ represents a bounded linear observation functional (e.g., $\observer_i(\model) = \model(\location_i)$, or $\observer_i(\model) = \model(\location_{i,1}) - \model(\location_{i,2})$), $\mlabel_i$ is given by utility evaluations (e.g., $\mlabel_i = \utility(\observation_i)$, or $\mlabel_{\nobs, i} = \utility(\observation_{i,1}) - \utility(\observation_{i,2})$), for $i\in\{1,\dots, \nobs\}$, and $\regfun_\nobs:\Hspace_\kernel\to\R$ is a regularization functional.

Considering the loss structure above, the target loss functional is given by:
\begin{equation}
    \Loss_*(\model) =\int \expectation[\loss(\observer(\model), \mlabel)] \basemeasure(\diff\observer).
\end{equation}
For the forward KL loss, for example, the target loss is simply $-\int_\domain \prior(\location)\af(\location) \model(\location) \diff\basemeasure(\location)$, where $\basemeasure$ is a suitable base measure over $\domain$ with respect to which the probability densities are defined (e.g., the counting measure if $\domain$ is discrete).
% where $\basemeasure(\diff\observer)$ represents the image of the base measure under the mapping by observers.
We then define constrained and unconstrained targets as:
\begin{align}
    \model_* &\in \argmin_{\model \in \Hspace_\kernel: \int e^\model \diff\basemeasure = 1} \Loss_*(\model)\\
    \tilde\model_* &\in \argmin_{\model \in \Hspace_\kernel} \Loss_*(\model).
\end{align}

\begin{assumption}[RKHS]
	\label{a:rkhs}
    There exists a reproducing kernel Hilbert space $\Hspace_\kernel$, associated with a positive-definite kernel $\kernel:\domain\times\domain\to\R$, which is bounded, $\sup_{\location\in\domain} \kernel(\location, \location) \leq \bound_\kernel^2$
	for a given $\bound_\kernel > 0$, such that $\model_*$, constants, and the models can be found as elements of $\Hspace_\kernel$, i.e., $\{\model(\cdot; \parameters) \mid \parameters\in\paramspace\} \subset \Hspace_\kernel$.
\end{assumption}

The assumption above allows us to consider functions $\model_*$ that cannot be perfectly approximated by the model, though which still lie in the same underlying Hilbert space $\Hspace_\kernel$. The reproducing kernel assumption is also mild, as it simply means that function evaluations are continuous (i.e., well behaved), which cannot usually be guaranteed in other types of Hilbert spaces, such as, e.g., $\lpspace{2}$-spaces. In fact, we can always find a RKHS that contains the set of models under mild assumptions, such as the minimal construction in \autoref{thr:model-rkhs}. The inclusion of constant functions allows for the unnormalized target $\tilde\model_*$ to also lie in $\Hspace_\kernel$.

\begin{remark}
If an RKHS containing the model class (e.g., $\Hspace_\model$ in \autoref{thr:model-rkhs}) is too small to contain $\model_*$, we can always combine two RKHS to produce a third one containing all the elements of the two. For instance, if $\model_*  \in \Hspace_* \neq \Hspace_\model$ with kernel $\kernel_*:\domain\times\domain\to\R$, we can define $\kernel:= \kernel_* + \kernel_\model$, so that $\Hspace_\kernel := \Hspace_* \oplus \Hspace_\model$ is also a RKHS \citep{Steinwart2008, Saitoh2016}. Such $\Hspace_*$ can be minimal, as any function $\model_*$ defines a kernel $\kernel_*(\location,\location') = \model_*(\location)\model_*(\location')$ for the RKHS formed by the function's linear span $\Hspace_* := \{\alpha \model_* \mid \alpha \in \R\}$. %Recall that the RKHS is for us only a theoretical construct.
\end{remark}

\begin{assumption}[Regularization]
	\label{a:regularization}
	The regularizer $\regfun_\nobs$ is predictable, $\regfactor_\nobs$-strongly convex, twice differentiable, and $\bigo(\regfactor_\nobs)$-smooth, for all $\nobs \geq 1$, where $\{\regfactor_\nobs\}_{\nobs=1}^\infty$ is a non-decreasing sequence of positive real numbers, such that $\regfactor_\nobs$ is at most poly-logarithmic with $\nobs$.
\end{assumption}

\paragraph{Regularization.} Common choices of regularization scheme, such as the squared norm, suffice to satisfy \autoref{a:regularization}. Strong convexity does not require a function to be twice differentiable, but such assumption greatly simplifies our analysis and is common in modern deep learning frameworks. Sublinearity of $\regfactor_\nobs$ allows for the effect of regularization to disappear as the dataset grows, so that the empirical $\Loss_\nobs$ can converge to the target $\Loss_*$ as $\nobs\to\infty$. Despite the definition of a regularization functional over the whole of $\Hspace_\kernel$, following \autoref{thr:model-norm}, we can use any positive-definite kernel $\kernel_\paramspace:\paramspace\times\paramspace\to\R$ compatible with \autoref{thr:model-rkhs} to set $\regfun_\nobs$ such that, over the model space:
\begin{equation}
	\regfun_\nobs(\model_\parameters) = \frac{\regfactor_\nobs}{2} \norm{\model_\parameters}_{\Hspace_\model}^2 = \frac{\regfactor_\nobs}{2}\kernel_\paramspace(\parameters,\parameters), %TODO: Check the gradient and Hessian of this constraint. We might need to remove the square and be careful with that.
\end{equation}
which allows for differentiation with respect to $\parameters$, and whose Hessian $\nabla_\model^2 \regfun_\nobs(\model) = \regfactor_\nobs \idop$ shows that $\regfun_\nobs$ is strongly convex in $\Hspace_\kernel$ for $\regfactor_\nobs > 0$.
In this case, a quadratic regularization penalty $\norm{\parameters}_2^2$ corresponds to the assumption of a linear kernel, i.e., %\footnote{Having $\kernel_\paramspace(\parameters,\parameters) = \norm{\parameters}_2^2$ corresponds to a quadratic kernel with constant term set to zero, $\anyconstant := 0$.}
$\kernel_\paramspace(\parameters,\parameters') = \parameters\cdot\parameters'$, which might appear quite restrictive, as it assumes that our models are linear functions of the parameters. However, note that, for overparameterized neural networks, at the infinite-width limit the model tends to show linearity in the parameters \citep{Jacot2018}. If we want to be more parsimonious, alternatively, we can choose $\kernel_\paramspace$ as a universal kernel, such as the squared exponential, yet preferably not translation invariant, so that $\kernel_\paramspace(\parameters,\parameters)$ is not simply a constant. One kernel satisfying such assumption would be the exponential dot-product kernel $\kernel_\paramspace(\parameters, \parameters') := \exp (\parameters\cdot\parameters')$, which is universal for continuous functions over compact subsets of $\paramspace$. %TODO: Justify/add citation. Exponentials can be expanded as infinite polynomials, and polynomials are dense in the set of analytic functions, which is itself dense in the set of continuous functions, as far as I understand.
Nevertheless, we do not impose restrictions on the form of the regularization term $\regfun_\nobs$ other than \autoref{a:regularization}. Lastly, $\regfun_\nobs$ can also be stochastic as:
\begin{equation}
    \regfun_\nobs(\model) = \frac{\regfactor_\nobs}{2}\norm{\model - \model_{\nobs,0}}_\kernel^2,
\end{equation}
where $\norm{\cdot}_\kernel$ denotes the norm in a general $\Hspace_\kernel$ satisfying \autoref{a:rkhs}, and $\model_{\nobs,0}$ may correspond to a random initialization or a data-dependent term (e.g., the previous set of optimal parameters).

% Common choices of regularization scheme, such as the squared norm $\norm{\model}_\kernel^2$, suffice the assumptions above. Strong convexity does not require a function to be twice differentiable, but such assumption can greatly simplify the analysis and it is common in modern deep learning frameworks.

\begin{assumption}[Loss]
    \label{a:loss}
    For any $\observation\in\R$, the point loss $\loss_\mlabel:=\loss(\cdot, \observation):\R\to\R$ is $\lossfactor$-strongly convex, twice differentiable, and has $\smoothness_\loss$-smooth first-order derivatives. In addition, given any $\observer\in\Hspace_\kernel$, we assume the first-order derivative $\dot{\loss}_\mlabel(\observer(\tilde\model_*))$ is conditionally $\sigma_\loss^2$-sub-Gaussian when $\observation\sim p(\observation|\observer)$.
\end{assumption}

Regarding sub-Gaussianity, the pointwise derivative of the balanced forward KL loss is given by $\dot\loss_\mlabel(\observer(\model)) = -\utility + e^{\model(\location)}$. Therefore, we have that:
\begin{equation}
    \expectation[\dot\loss_\mlabel(\observer(\model)) \mid \location] = e^{\model(\location)} - \af(\location),
\end{equation}
which yields $\tilde\model_*(\location) = \log \af(\location)$ at the unconstrained minimizer $\tilde\model_*$ under mild assumptions on the kernel and the acquisition function. Therefore, if observation noise is light tailed or the utilities are bounded, the pointwise derivatives can shown to be sub-Gaussian. The second derivative is simply $e^{\model(\location)}$, which is strictly positive, yielding an approximate strong convexity (when restricted to any bounded subset of the realizations of $\model(\location)$). The original Bradley-Terry model in the DPO paper \citep{Rafailov2023dpo} is not strongly convex, but its robust version \citep{Chowdhury2024rdpo}, which accounts for preference noise, can be shown to satisfy strong convexity and smoothness.

\begin{assumption}[Weights]
    \label{a:weights}
    We have that $\inf_{\nobs\in\N} \weight_\nobs \geq \bound_\weight$ and $\sup_{\nobs\in\N}\weight_\nobs \leq \bar\bound_\weight$ almost surely, for constants $\bar\bound_\weight, \bound_\weight > 0$.
\end{assumption}

\begin{assumption}[Density]
    \label{a:dense}
    The model class is dense in the normalized subset of $\Hspace_\kernel$:
    \begin{equation*}
        \set{C} := \left\lbrace \model \in \Hspace_\kernel \mmiddle| \int e^{\model} \diff\basemeasure = 1 \right\rbrace.
    \end{equation*}
\end{assumption}

These last two assumptions allow us to bound the randomness of the loss gradients and ensure that the optimal models converge to the target distribution. Density here simply means that, for any $\model \in \set{C}$, we can find a sequence $\{\parameters_i\}_{i=1}^\infty$ such that $\model_{\parameters_i} \to \model$ as $i\to\infty$. Alternatively, the latter can be expressed as, for any $\model\in\set{C}$ and any given $\gap > 0$, there exists $\parameters_\gap\in\paramspace$ such that $\norm{\model - \model_{\parameters_\gap}}_\kernel \leq \gap$. Hence, density is a universal approximation condition within the normalized subset of $\Hspace_\kernel$. As $\Hspace_\kernel$ is just a theoretical construct for us, and not fixed by the algorithm, so that we can even span $\Hspace_\kernel$ from the model class (see \autoref{thr:model-rkhs}), \autoref{a:dense} is relatively mild.

We can now analyze the approximation error with respect to $\model_*$ for the following estimators:\footnote{We are implicitly assuming that such global optima exist. This is true for the optimization in $\Hspace_\kernel$, as $\Loss_\nobs$ is strongly convex over it, but that is not guaranteed for the optimization over the parameter space $\paramspace$.}
\begin{align}
    \parameters_\nobs &\in \argmin_{\parameters\in\paramspace} \Loss_\nobs(\model_\parameters) \\
    \model_\nobs &\in \argmin_{\model\in\Hspace_\kernel} \Loss_\nobs(\model)\,.
\end{align}
The first one gives us the best parametric approximation $\model_{\parameters_\nobs}$ based on the data and is what our algorithm will use. The second estimator corresponds to the non-parametric approximation, which we will use as a tool for our analysis, and not assume as a component of the algorithm. The assumptions above allow us to bound distances between these estimators and the true $\model_*$ as a function of the loss and gradient values. We also consider the normalized version of $\model_\nobs$ as:
\begin{equation}
    \hat\model_\nobs(\location) := \model_\nobs(\location) - \log \int e^{\model_\nobs} \diff\basemeasure, \quad \location\in\domain,
    \label{eq:normed-minimizer}
\end{equation}
which we use as a reference point to bound distances to the parametric models and to the target $\model_*$.

\subsection{Main results}

\lossbounds*
\begin{proof}[Proof of \autoref{thr:loss-bounds}]
	The Hessian of the individual losses can be lower bounded by:
	\begin{equation}
		\begin{split}
			\forall\model\in\Hspace_\kernel, \qquad 
			\nabla_\model^2 \loss(\observer(\model), \mlabel)
			&= \ddot{\loss}_\mlabel(\observer(\model)) \nabla_\model \observer(\model) \otimes \nabla_\model \observer(\model) + \dot{\loss}_\mlabel(\observer(\model)) \nabla_\model^2 \observer(\model)\\
			&= \ddot{\loss}_\mlabel(\observer(\model)) \observer \otimes \observer\\
			&\succeq \lossfactor\observer \otimes \observer, \qquad \forall \mlabel \in \R,\, \forall\observer\in\observerset,
		\end{split}
	\end{equation}
	given that $\nabla_\model \observer(\model) = \nabla_\model \inner{\model, \observer}_\kernel = \observer$. Combining this result with \autoref{a:regularization} and \ref{a:weights}, we get:
	\begin{equation}
		\forall \model \in \Hspace_\kernel, \qquad \nabla_\model^2 \Loss_\nobs(\model) \succeq \regfactor_\nobs\idop + \lossfactor\bound_\weight\sum_{i=1}^\nobs \observer_i \otimes \observer_i =: \Hessian_\nobs \,.
		\label{eq:loss-hessian}
	\end{equation}
	Now applying a first order Taylor expansion to $\Loss_\nobs$ at any $\model\in\Hspace_\kernel$, the error term is controlled by the Hessian $\nabla^2\Loss_\nobs(\bar\model_\nobs)$ at an intermediate point $\bar\model_\nobs = \anyscalar\model_\nobs + (1-\anyscalar)\model$, for some $\anyscalar \in [0,1]$. Hence, expanding $\Loss_\nobs$ around $\model_\nobs$, we have that:
	\begin{equation}
		\begin{split}
			\forall \model \in \Hspace_\kernel, \quad \Loss_\nobs(\model) - \Loss_\nobs(\model_\nobs)
			&= \inner{\nabla \Loss_\nobs(\model_\nobs), \model - \model_\nobs} + \frac{1}{2} \norm{\model - \model_\nobs}_{\nabla^2 \Loss_\nobs(\bar\model_\nobs)}^2\\
			&\geq  \frac{1}{2} \norm{\model - \model_\nobs}_{\Hessian_\nobs}^2\,,
		\end{split}
	\end{equation}
	where we applied the Hessian inequality \eqref{eq:loss-hessian} and the fact that $\nabla\Loss_\nobs(\model_\nobs) = 0$, as $\model_\nobs$ is a minimizer. Thus, the lower bound in \autoref{thr:loss-bounds} follows. Conversely, expanding $\Loss_\nobs$ around any $\model$ and evaluating at $\model_\nobs$, we have:
	\begin{equation}
		\begin{split}
			\forall \model \in \Hspace_\kernel, \quad \Loss_\nobs(\model_\nobs) &= \Loss_\nobs(\model)
			+ \inner{\nabla \Loss_\nobs(\model), \model_\nobs - \model}_\kernel + \frac{1}{2} \norm{\model_\nobs - \model}_{\nabla^2 \Loss_\nobs(\bar\model_\nobs')}^2.
		\end{split}
	\end{equation}
	Rearranging terms yields:
	\begin{equation}
		\begin{split}
			\forall \model \in \Hspace_\kernel, \quad \Loss_\nobs(\model) - \Loss_\nobs(\model_\nobs)
			&= \inner{\nabla \Loss_\nobs(\model), \model - \model_\nobs}_\kernel - \frac{1}{2} \norm{\model - \model_\nobs}_{\nabla^2 \Loss_\nobs(\bar\model_\nobs')}^2\\
			&\leq \sup_{\bar\model \in \Hspace_\kernel}  \inner{\nabla \Loss_\nobs(\model), \bar\model}_\kernel - \frac{1}{2} \norm{\bar\model}_{\nabla^2 \Loss_\nobs(\bar\model_\nobs')}^2\\
			&\leq \sup_{\bar\model \in \Hspace_\kernel}  \inner{\nabla \Loss_\nobs(\model), \bar\model}_\kernel - \frac{1}{2} \norm{\bar\model}_{\Hessian_\nobs}^2\,,
		\end{split}
	\end{equation}
	whose right-hand side is strongly concave and has therefore a unique maximizer at:
	\begin{equation}
		\bar\model = \Hessian_\nobs^{-1} \nabla \Loss_\nobs(\model)\,.
	\end{equation}
	Replacing this result into the previous equation finally leads us to the upper bound in \autoref{thr:loss-bounds}.
\end{proof}

\begin{lemma}
    \label{thr:error-unconstrained}
    Consider the setting in \autoref{thr:loss-bounds}. Then,
    \begin{equation*}
        \forall\nobs\in\N, \quad \lvert \inner{\observer, \tilde{\model}_*}_\kernel - \inner{\observer, \model_\nobs}_\kernel \rvert 
        \leq 
        \norm{\observer}_{\Hessian_\nobs^{-1}}\beta_\nobs(\delta),
        \: \forall \observer\in\Hspace_\kernel,
    \end{equation*}
    where $\beta_\nobs(\delta) := \regfactor_\nobs^{-1/2}\norm{\nabla\regfun_\nobs(\tilde{\model}_*)}_\kernel + \sigma_\loss\bar\bound_\weight \sqrt{2(\bound_\weight\lossfactor)^{-1}\log(\det(\eye + \bound_\weight\lossfactor\regfactor_\nobs^{-1}\observers_\nobs^\transpose\observers_\nobs)^{1/2}/\delta)}$, and $\observers_\nobs := [\observer_1, \dots, \observer_\nobs]$.
\end{lemma}
\begin{proof}
    By \autoref{thr:loss-bounds}, for any $\model\in\Hspace_\kernel$, we have that:
    \begin{equation}
        \begin{split}
            \lvert \inner{\observer, \model_\nobs}_\kernel - \inner{\observer, \model}_\kernel \rvert
            &= \lvert \inner{\observer, \model_\nobs - \model}_\kernel \rvert\\
            &= \lvert \inner{\Hessian_\nobs^{-1/2}\observer, \Hessian_\nobs^{1/2}(\model_\nobs - \model)}_\kernel \rvert\\
            &\leq \norm{\Hessian_\nobs^{-1/2}\observer} \norm{\Hessian_\nobs^{1/2}(\model_\nobs - \model)}\\
            &= \norm{\observer}_{\Hessian_\nobs^{-1}} \norm{\model_\nobs - \model}_{\Hessian_\nobs}\\
            &\leq \norm{\observer}_{\Hessian_\nobs^{-1}} \norm{\nabla\Loss_\nobs(\model)}_{\Hessian_\nobs^{-1}}\,, \quad \forall\observer\in\Hspace_\kernel,
        \end{split}
        \label{eq:error-decomp}
    \end{equation}
    where the first inequality follows by Cauchy-Schwarz, and the last is due to \autoref{thr:loss-bounds}. Expanding the gradient term, we have:
    \begin{equation}
        \begin{split}
            \norm{\nabla\Loss_\nobs(\model)}_{\Hessian_\nobs^{-1}}
            &= \flexnorm{
                \nabla\regfun_\nobs(\model) + \sum_{i=1}^\nobs \weight_i \dot{\loss}_{\observation_i}(\inner{\observer_i, \model}_\kernel) \observer_i
            }_{\Hessian_\nobs^{-1}}\\
            &\leq \norm{\nabla\regfun_\nobs(\model)}_{\Hessian_\nobs^{-1}} 
            + \flexnorm{
                \sum_{i=1}^\nobs \weight_i \dot{\loss}_{\observation_i}(\inner{\observer_i, \model}_\kernel) \observer_i
            }_{\Hessian_\nobs^{-1}}\\
            &\leq \frac{1}{\sqrt{\regfactor}}\norm{\nabla\regfun_\nobs(\model)}_\kernel
            + \flexnorm{
                \sum_{i=1}^\nobs \weight_i \dot{\loss}_{\observation_i}(\inner{\observer_i, \model}_\kernel) \observer_i
            }_{\Hessian_\nobs^{-1}},
        \end{split}
        \label{eq:grad-bound}
    \end{equation}
    where we applied the triangle inequality to obtain the second line and the fact that $\Hessian_\nobs \succ \regfactor_\nobs\idop$ implies $\Hessian_\nobs^{-1} \prec \regfactor_\nobs^{-1}\idop$ led to the last line. For $\model := \tilde{\model}_*$, we can then apply \autoref{thr:noise-bound} to the noisy sum above by setting $\filtration_\iteridx$ as the $\sigma$-algebra generated by the random variables $\{\observer_i, \observation_i\}_{i=1}^\iteridx$ and $\observer_{\iteridx+1}$. Then $\obsnoise_\iteridx := \frac{1}{\sqrt{\bound_\weight\lossfactor}}\weight_\iteridx\dot{\loss}_{\observation_\iteridx}(\inner{\tilde{\model}_*, \observer_\iteridx}_\kernel)$ is conditionally $\sigma_\obsnoise^2$-sub-Gaussian by \autoref{a:loss} with $\sigma_\obsnoise^2 := \frac{\bar\bound_\weight^2 \sigma_\loss^2}{\bound_\weight \lossfactor}$, since $\weight_\iteridx \leq \bar\bound_\weight$ by \autoref{a:weights}, and $\feature_\iteridx := \sqrt{\bound_\weight\lossfactor}\observer_\iteridx$ is predictable, for all $\iteridx\in\N$. An application of \autoref{thr:noise-bound} finally leads us to:
    \begin{equation}
        \flexnorm{
                \sum_{i=1}^\nobs \dot{\loss}_{\observation_i}(\inner{\observer_i, \tilde{\model}_*}_\kernel) \observer_i
            }_{\Hessian_\nobs^{-1}}^2
        \leq \frac{2\bar\bound_\weight^2\sigma_\loss^2}{\bound_\weight\lossfactor} \log \left(
                \frac{
                    \det(\eye + \bound_\weight\lossfactor\regfactor_\nobs^{-1}\observers_\nobs^\transpose \observers_\nobs)^{\frac{1}{2}}
                }{\delta}
            \right)
        \label{eq:loss-noise-bound}
    \end{equation}
    which holds uniformly over all $\nobs\in\N$ with probability at least $1-\delta$. Hence, it follows that:
    \begin{equation}
        \begin{split}
            \forall\nobs\in\N, \quad 
            \norm{\nabla\Loss_\nobs(\tilde{\model}_*)}_{\Hessian_\nobs^{-1}}
            &\leq \frac{1}{\sqrt{\regfactor}} \norm{\nabla\regfun_\nobs(\tilde{\model}_*)}_\kernel
            + \flexnorm{
                \sum_{i=1}^\nobs \weight_i\dot{\loss}_{\observation_i}(\inner{\observer_i, \tilde{\model}_*}_\kernel) \observer_i
            }_{\Hessian_\nobs^{-1}}
            \leq \beta_\nobs(\delta),
        \end{split}
        \label{eq:rkhs-optimal-grad-bound}
    \end{equation}
    with probability at least $1-\delta$. Replacing this result into \autoref{eq:error-decomp} yields the main result.
\end{proof}

\mainthm*
\begin{proof}[Proof of \autoref{thr:error-bound}]
    Fix any $\observer\in\Hspace_\kernel$ and $\nobs\in\N$, the approximation error can then be decomposed as:
    \begin{equation}
        \begin{split}
            \lvert \inner{\observer, \model_*}_\kernel - \inner{\observer, \model_{\parameters_\nobs}}_\kernel \rvert
            &\leq \lvert \inner{\observer, \model_* - \hat\model_\nobs}_\kernel  \rvert
            + \lvert \inner{\observer, \model_{\parameters_\nobs} - \hat\model_\nobs}_\kernel \rvert.\\
            % &\leq \norm{\observer}_{\Hessian_\nobs^{-1}} (\norm{\model_* - \hat\model_\nobs}_{\Hessian_\nobs} 
            % + \norm{\model_{\parameters_\nobs} - \hat\model_\nobs}_{\Hessian_\nobs}),
        \end{split}
        \label{eq:error-split}
    \end{equation}
    Let $\unitf\in\Hspace_\kernel$ denote the unit constant function, i.e., $\unitf(\location) = 1$, for all $\location\in\domain$. The difference between normalized functions is then such that:
    \begin{equation}
        \begin{split}
            |\inner{\observer, \model_* - \hat\model_\nobs}_\kernel|
            &= \left\lvert \flexinner{\observer, \tilde\model_* - \model_\nobs + \unitf \left( \log \int e^{\model_\nobs} \diff\basemeasure - \log \int e^{\tilde\model_*} \diff\basemeasure \right)}_\kernel \right\rvert\\
            &\leq |\inner{\observer, \tilde\model_* - \model_\nobs}_\kernel| 
            + |\inner{\observer, \unitf}_\kernel| \left\lvert \log \int e^{\model_\nobs} \diff\basemeasure - \log \int e^{\tilde\model_*} \diff\basemeasure \right\rvert
            % &\leq \norm{\tilde\model_* - \model_\nobs}_{\Hessian_\nobs}
            % + \norm{\unitf}_{\Hessian_\nobs} \left\lvert \log \int e^{\model_\nobs} \diff\basemeasure - \log \int e^{\tilde\model_*} \diff\basemeasure  \right\rvert.
            \label{eq:error-bound-normed-functions}
        \end{split}
    \end{equation}
    We can apply Jensen's inequality to show that the second term on the right-hand side of the inequality is bounded by the expected pointwise error between $\model_\nobs$ and $\tilde\model_*$. To reduce notation clutter, let $\proposal_* := p_\utility^* = \exp \model_*$ and $\Ex[\proposal]{\anyfunction} := \Ex[\location\sim\proposal]{\anyfunction(\location)}$. Indeed, therefore, we have that:
    \begin{equation}
        \begin{split}
            \Ex[\proposal_*]{\tilde\model_* - \model_\nobs}
            &= \Ex[\proposal_*]{-\log\left(\frac{e^{\model_\nobs}}{e^{\tilde\model_*}}\right)}\\
            &\geq -\log \Ex[\proposal_*]{\frac{e^{\model_\nobs}}{e^{\tilde\model_*}}}\\
            &= -\log \left( \frac{1}{\int e^{\tilde\model_*}\diff\basemeasure} \int e^{\tilde\model_*} \frac{e^{\model_\nobs}}{e^{\tilde\model_*}} \diff\basemeasure \right)\\
            &= \log \int e^{\tilde\model_*}\diff\basemeasure - \log \int e^{\model_\nobs} \diff\basemeasure
        \end{split}
    \end{equation}
    by Jensen's inequality on the convex $-\log(\cdot)$. Similarly, repeating the same steps for $\Ex[\hat\proposal_\nobs]{\model_\nobs - \tilde\model_*}$, where $\hat\proposal_\nobs(\location) = e^{\hat\model_\nobs(\location)} = \frac{e^{\model_\nobs(\location)}}{\int e^{\model_\nobs} \diff\basemeasure}$, we get:
    \begin{equation}
        \Ex[\hat\proposal_\nobs]{\model_\nobs - \tilde\model_*} \geq \log \int e^{\model_\nobs} \diff\basemeasure - \log \int e^{\tilde\model_*}\diff\basemeasure.
    \end{equation}
    Combining the two inequalities yields:
    \begin{equation}
        \begin{split}
            \left\lvert \log \int e^{\model_\nobs} \diff\basemeasure - \log \int e^{\tilde\model_*} \diff\basemeasure  \right\rvert
            &\leq
            \max\{ \Ex[\proposal_*]{\tilde\model_* - \model_\nobs},  \Ex[\hat\proposal_\nobs]{\model_\nobs - \tilde\model_*}\}
            \\
            &\leq |\Ex[\proposal_*]{\tilde\model_* - \model_\nobs}| + |\Ex[\hat\proposal_\nobs]{\model_\nobs - \tilde\model_*}|.
        \end{split}
    \end{equation}
    We can now bound the expectations as:
    \begin{equation}
        \begin{split}
            \forall \proposal \in \Pspace(\domain), \quad |\Ex[\proposal]{\tilde\model_* - \model_\nobs}|
            &= |\Ex[\location\sim\proposal]{\tilde\model_*(\location) - \model_\nobs(\location)}|\\
            &= |\Ex[\location\sim\proposal]{\inner{\tilde\model_* - \model_\nobs, \feature(\location)}_\kernel}|\\
            &= |\inner{\tilde\model_* - \model_\nobs, \Ex[\location\sim\proposal]{\feature(\location)}}_\kernel|\\
            &\leq \beta_\nobs(\delta) \norm{\Ex[\proposal]{\feature(\location)}}_{\Hessian_\nobs^{-1}}\\
            &\leq \beta_\nobs(\delta) \Ex[\proposal]{\norm{\feature(\location)}_{\Hessian_\nobs^{-1}}}.
            % &= |\Ex[\location\sim\proposal]{\inner{\Hessian_\nobs^{1/2}(\tilde\model_* - \model_\nobs), \Hessian_\nobs^{-1/2}\feature(\location)}_\kernel}|\\
            % &\leq \norm{\tilde\model_* - \model_\nobs}_{\Hessian_\nobs} \Ex[\proposal]{\norm{\feature(\location)}_{\Hessian_\nobs^{-1}}},
        \end{split}
    \end{equation}
    where the first inequality follows by \autoref{thr:error-unconstrained} and the second is due to Jensen's.
    Therefore,
    \begin{equation}
        \left\lvert \log \int e^{\model_\nobs} \diff\basemeasure - \log \int e^{\tilde\model_*} \diff\basemeasure  \right\rvert
        \leq \beta_\nobs(\delta)
        \left(\Ex[\proposal_*]{\norm{\feature(\location)}_{\Hessian_\nobs^{-1}}} 
        + \Ex[\hat\proposal_\nobs]{\norm{\feature(\location)}_{\Hessian_\nobs^{-1}}} \right).
    \end{equation}
    % Lastly, by Cauchy-Schwarz, we have that:
    % \begin{equation}
    %     |\inner{\observer, \tilde\model_* - \model_\nobs}_\kernel| = |\inner{\Hessian_\nobs^{-1/2}\observer, \Hessian_\nobs^{1/2}(\tilde\model_* - \model_\nobs)}_\kernel|
    %     \leq \norm{\observer}_{\Hessian_\nobs^{-1}} \norm{\tilde\model_* - \model_\nobs}_{\Hessian_\nobs}.
    % \end{equation}
    Finally, applying the bound above to \autoref{eq:error-bound-normed-functions} followed by another application of \autoref{thr:error-unconstrained}, we obtain:
    \begin{equation}
        |\inner{\observer, \model_* - \hat\model_\nobs}_\kernel| \leq \beta_\nobs(\delta)
        \left( 
        \norm{\observer}_{\Hessian_\nobs^{-1}} 
        + |\inner{\observer, \unitf}_\kernel| 
            \left(
            \Ex[\proposal_*]{\norm{\feature(\location)}_{\Hessian_\nobs^{-1}}} 
            + \Ex[\hat\proposal_\nobs]{\norm{\feature(\location)}_{\Hessian_\nobs^{-1}}} 
            \right)
        \right).
        \label{eq:error-constrained}
    \end{equation}
    
    For the remaining term in \autoref{eq:error-split}, we have that:
    \begin{equation}
        \begin{split}
            \lvert \inner{\observer, \model_{\parameters_\nobs}}_\kernel - \inner{\observer, \hat\model_\nobs}_\kernel \rvert
            &\leq \norm{\observer}_{\Hessian_\nobs^{-1}} \norm{\model_{\parameters_\nobs} - \hat\model_\nobs}_{\Hessian_\nobs}\\
            &\leq \norm{\observer}_{\Hessian_\nobs^{-1}}
            \sqrt{
                2(\Loss_\nobs(\model_{\parameters_\nobs}) - \Loss_\nobs(\hat\model_\nobs)) 
            },
        \end{split}
        \label{eq:model-error-decomp}
    \end{equation}
    which follows from \autoref{thr:loss-bounds}. From \autoref{a:dense}, we have that:
    \begin{equation}
        \forall \gap > 0, \quad \exists \parameters_\gap \in \paramspace: \qquad \norm{\model_{\parameters_\gap} - \hat\model_\nobs}_\kernel \leq \gap\,.
    \end{equation}
    At the optimum, we know that $\Loss_\nobs(\model_{\parameters_\nobs}) \leq \Loss_\nobs(\model_\parameters)$, for all $\parameters\in\paramspace$. Therefore, as $\gap\to 0$, by continuity, we have that:
    % \begin{equation}
    %     \Loss
    % \end{equation}
    % Therefore, as $\parameters_\nobs$ minimizes $\Loss_\nobs$ over all $\paramspace$, picking some $\gap > 0$, we have that any $\parameters_\gap$ satisfying the condition above leads us to:
    \begin{equation}
        \begin{split}
            \Loss_\nobs(\model_{\parameters_\nobs}) - \Loss_\nobs(\hat\model_\nobs)
            &\leq \Loss_\nobs(\model_{\parameters_\gap}) - \Loss_\nobs(\hat\model_\nobs) \to 0, %\\
            % &\leq \frac{1}{2}\norm{\nabla \Loss_\nobs(\model_{\parameters_\gap})}_{\Hessian_\nobs^{-1}}^2\\
            % &\leq \frac{1}{2}\left( \norm{\nabla\bar\regfun_\nobs(\model_{\parameters_\gap})}_\kernel + 
            %     \flexnorm{
            %         \sum_{i=1}^\nobs \dot{\loss}_{\observation_i}(\inner{\observer_i, \model_{\parameters_\gap}}_\kernel) \observer_i
            %     }_{\Hessian_\nobs^{-1}}
            % \right)^2\,,
        \end{split}
        \label{eq:loss-gap}
    \end{equation}
    which implies $\lvert\inner{\observer, \model_{\parameters_\nobs}}_\kernel - \inner{\observer, \hat\model_\nobs}_\kernel \rvert\to 0$ in \autoref{eq:model-error-decomp}. Consequently, the density of the model class allows us to replace pointwise evaluations $\hat\model_\nobs(\location) = \inner{\hat\model_\nobs, \feature(\location)}_\kernel$ by $\model_{\parameters_\nobs}(\location)$, for $\location\in\domain$, so that we can swap $\hat\proposal_\nobs$ in \autoref{eq:error-constrained} for $\proposal_\nobs(\location) = \exp\model_{\parameters_\nobs}(\location)$.
    The main result then follows.
\end{proof}

% \begin{corollary}
%     \label{thr:limiting-cov}
%     Consider the setting in \autoref{thr:error-bound}, and further assume that $\frac{1}{\nobs}\Hessian_\nobs \to \vncov_\infty \succ 0$ almost surely. Then, with probability at least $1-\delta$, for $\delta\in(0,1]$,
%     \begin{equation*}
%         \lvert \inner{\observer, \model_*}_\kernel - \inner{\observer, \model_{\parameters_\nobs}}_\kernel \rvert 
%         \leq 
%         \norm{\observer}_{\Hessian_\nobs^{-1}} \beta_\nobs(\delta)\anyconstant_\kernel\sqrt{\condnum(\vncov_\infty)},
%         \quad \forall \observer\in\Hspace_\kernel,
%     \end{equation*}
%     for all $\nobs$ sufficiently large.
% \end{corollary}
% \begin{proof}
%     Asymptotically, we have that:
%     \begin{equation}
%         \condnum(\Hessian_\nobs) = \frac{\eigval_{\max}(\Hessian_\nobs)}{\eigval_{\min}(\Hessian_\nobs)} \to \frac{\eigval_{\max}(\nobs \vncov_\infty)}{\eigval_{\min}(\nobs\vncov_\infty)} = \frac{\eigval_{\max}(\vncov_\infty)}{\eigval_{\min}(\vncov_\infty)} < \infty,
%     \end{equation}
%     where $\eigval_{\max}$ and $\eigval_{\min}$ represent the maximum and minimum eigenvalues, respectively.
% \end{proof}

Despite the model being potentially non-linear and the loss not being required to be least-squares, \autoref{thr:error-bound} shows that we recover the same kind of RKHS-based error bound found in the kernelized bandits literature \citep{Chowdhury2017, Durand2018, Oliveira2021}. Regarding the asymptotic rates, we make the following observations.

\begin{remark}
\label{thr:vanishing-error}
 % If the (uncentered) empirical covariance operator $\frac{1}{n}\sum_{i=1}^\nobs \observer_i \otimes \observer_i$ converges to a positive-definite limit, \autoref{thr:limiting-cov} further shows us that the condition number appearing in \autoref{thr:error-bound} is asymptotically bounded.
% If identifiability holds, we have $\error_\kernel = 0$, and we recover the usual bounds \citep{Durand2018}.\footnote{If we further assume that the model can represent \emph{any} $\model \in \Hspace_\kernel$, the factor of 2 multiplying $\beta_\nobs$ would also disappear, as the extra $\beta_\nobs$ arises from a bound over $\norm{\model_{\parameters_\nobs} - \model_\nobs}$, which would vanish.} Alternatively, in the case of neural networks, we can increase the width of the network over time (making sure the model scales up with the data is not uncommon in deep learning approaches), which would then lead to the model covering a whole RKHS, determined by the NTK \citep{Jacot2018}. In general, for a rich enough model class, one may expect $\error_\kernel$ to be small.
% Regarding the term $\norm{\observer}_{\Hessian_\nobs^{-1}}$, the same condition in \autoref{thr:limiting-cov} also ensures that:
% \begin{equation}
%     \norm{\observer}_{\Hessian_\nobs^{-1}} \to \frac{1}{\sqrt{\nobs}}\norm{\observer}_{\vncov_\infty^{-1}} \in \bigo(\nobs^{-1/2}).
% \end{equation}
For a finite domain, $\card{\domain} < \infty$, the RKHS becomes finite dimensional with $\dim(\Hspace_\kernel) = \card{\domain}$, resembling a linear model. In this case, \citet{Srinivas2010} provides bounds for the Gram matrix log-determinant in $\beta_\nobs(\delta)$ as $\bigo(\card{\domain}\log\nobs)$, yielding:
\begin{equation*}
    \beta_\nobs(\delta) \in \bigo(\sqrt{\card{\domain}\log\nobs}).
\end{equation*}
% Consequently, the approximation error in \autoref{thr:error-bound} is bounded by $\bigo(\sqrt{\nobs^{-1}\card{\domain}\log(\nobs)})$, which vanishes as $\nobs\to\infty$.
\end{remark}

% \begin{remark}
%     If the condition for a globally positive-definite limiting operator $\vncov_\infty$ in \autoref{thr:limiting-cov} holds, \autoref{thr:vanishing-error} shows us that $\model_{\parameters_\nobs} \to \model_*$ as $\nobs\to\infty$. As $\observer_\nobs$ will be sampled according to $p_\utility^* = \exp \model_*$ in the limit in this case, $\vncov_\infty$ corresponds to the (uncentered) covariance operator of the observation functionals when sampled according to $p_\utility^*$. We also note that the globally positive-definite condition on $\vncov_\infty$ can be relaxed, as the condition number bound in \autoref{eq:cond-number-origin} only needs to hold on the combined limiting linear span of $\model_*$ and $\{\hat\model_\nobs\}_{\nobs=1}^\infty$, The latter is a milder condition, possibly allowing for $\vncov_\infty$ to be rank deficient, as is the case when $p_\utility^*$ concentrates its support on a single point (e.g., the global optimum). However, we leave such analysis as subject of future work.
% \end{remark}

\paragraph{Predictive variance.} We can show a connection between $\norm{\observer}_{\Hessian_\nobs^{-1}}$ and a GP predictive variance. By an application of Woodbury's identity, letting $\widehat\lossfactor := \lossfactor\bound_\weight$, we have that:
\begin{equation}
    \begin{split}
        \norm{\observer}_{\Hessian_\nobs^{-1}}^2 &= \observer^\transpose (\regfactor_\nobs\idop + \widehat\lossfactor \observers_\nobs \observers_\nobs^\transpose )^{-1} \observer\\
        &= \observer^\transpose (\regfactor_\nobs^{-1}\idop - \regfactor_\nobs^{-2}\observers_\nobs(\widehat\lossfactor^{-1}\eye + \regfactor_\nobs^{-1}\observers_\nobs^\transpose\observers_\nobs)^{-1}\observers_\nobs^\transpose) \observer\\
        &= \regfactor_\nobs^{-1}\observer^\transpose (\idop - \observers_\nobs(\regfactor_\nobs\widehat\lossfactor^{-1}\eye + \observers_\nobs^\transpose\observers_\nobs)^{-1}\observers_\nobs^\transpose) \observer\\
        &= \regfactor_\nobs^{-1} (\norm{\observer}_\kernel^2 - \observer^\transpose\observers_\nobs(\regfactor_\nobs\widehat\lossfactor^{-1}\eye + \observers_\nobs^\transpose\observers_\nobs)^{-1}\observers_\nobs^\transpose\observer)\,,
    \end{split} 
\end{equation}
If observations correspond to pointwise evaluations $\observer := \kernel(\cdot, \location) = \feature(\location)$ and $\observer_i := \kernel(\cdot, \location_i) = \feature(\location_i)$, for $\location \in \domain$ and $\{\location_i\}_{i=1}^\nobs \subset\domain$, we end up with:
\begin{equation}
    \begin{split}
        \norm{\observer}_{\Hessian_\nobs^{-1}}^2
        &= \norm{\feature(\location)}_{\Hessian_\nobs^{-1}}^2\\
        &= \regfactor_\nobs^{-1} (\kernel(\location,\location) - \vec\kernel_\nobs(\location)^\transpose (\regfactor_\nobs\widehat\lossfactor^{-1}\eye + \Kernel_\nobs)^{-1}\vec\kernel_\nobs(\location))\\
        &= \regfactor_\nobs^{-1}\sigma_\nobs^2(\location)\,,
    \end{split}
\end{equation}
which corresponds to a scaled version of the posterior predictive variance $\sigma_\nobs^2(\location)$ of a GP \eqref{eq:gp-post-cov}. Moreover, the $\inner{\observer,\unitf}_\kernel$ term in \autoref{thr:error-bound} is simply $\inner{\feature(\location), \unitf}_\kernel = \unitf(\location) = 1$, for $\location\in\domain$, and the remaining terms depend on expectations of $\norm{\feature}_{\Hessian_\nobs^{-1}}$, which is also a function of the predictive variance $\sigma_\nobs^2(\location)$. We can then apply the auxiliary result from VSD (\autoref{thr:gp-variance-convergence}) to show that $\sigma_\nobs^2(\location)$ is $\bigo(\nobs^{-1})$ asymptotically, whenever $\proposal_{\theta_\nobs}(\location)$ has a positive lower bound, allowing for asymptotic convergence of the proposal distributions to the target $p_\utility^* = \exp \model_*$. As a result, considering \autoref{thr:vanishing-error}, we get:
\begin{equation}
    |\model_*(\location) - \model_\nobs(\location)| \in \bigo\left(\sqrt{\frac{\card{\domain}\log\nobs}{\nobs}}\right),
\end{equation}
which vanishes as $\nobs\to\infty$. Therefore, the model converges to the target distribution $p_\utility^*$ under these assumptions. Nevertheless, the algorithm's regret still depends on how fast the time-dependent acquisition function $\af_\iteridx$ concentrates around the optimum $\location^*$ and how much probability mass $\proposal_\iteridx$ places away from $\location^*$ per BO round $\iteridx$, analyses which we leave as subject of future work.

\subsection{Loss functions}
In the following, we present further details on the derivation of the loss functions. We consider a fixed non-negative utility function $\utility: \R \to [0,\infty)$ and the corresponding acquisition function $\af:\domain\to \R$ defined as $\af(\location) = \expectation[\utility(\observation)|\location]$, where the expectation is over the observations distribution.

% \clearpage

\paragraph{Forward KL.} Expanding the definition, we have:
\begin{equation}
    \begin{split}
        \kl{p_\utility^*}{\proposal}
        &= \expectation_{\location\sim p_\utility^*}[
            \log p_\utility^*(\location) - \log \proposal(\location)
        ]\\
        &= \expectation_{\location\sim p_\utility^*}[
            \log p_\utility^*(\location)
            ]
            -\expectation_{\location\sim p_\utility^*}[
            \log \proposal(\location)
            ]\\
        &= -\entropy{p_\utility^*}
        - \expectation_{\location\sim p_\utility^*}[
            \log \proposal(\location)
        ].
    \end{split}
\end{equation}
Note that the entropy $\entropy{p_\utility^*}$ of the target is constant. Using importance sampling with a proposal $\proposal_0$, the remaining expectation can be approximated as:
\begin{equation}
    \expectation_{\location\sim p_\utility^*}[
            \log \proposal(\location)
        ]
    = \expectation_{\location\sim \proposal_0}\left[
            \frac{p_\utility^*(\location)}{\proposal_0(\location)}
            \log \proposal(\location)
        \right]
    \approx \frac{1}{\nbatch}\sum_{i=1}^\nbatch 
    \frac{p_\utility^*(\location_i)}{\proposal_0(\location_i)}
            \log \proposal(\location_i),
\end{equation}
for a batch of $\nbatch \geq 1$ samples $\location_i \sim \proposal_0$ sampled \iid from the proposal $\proposal_0$. Having $\iteridx\geq 1$ proposals, instead, we get:
\begin{equation}
    \expectation_{\location\sim p_\utility^*}[
            \log \proposal(\location)
        ]
    \approx \frac{1}{\iteridx\nbatch}\sum_{j=1}^\iteridx \sum_{i=1}^\nbatch 
    \frac{p_\utility^*(\location_{j,i})}{\proposal_{j-1}(\location_{j,i})}
            \log \proposal(\location_{j,i}),
\end{equation}
where $\{\location_{j,i}\}_{i=1}^\nbatch \overset{\iid}{\sim} \proposal_{j-1}$, for $j \in \{1,\dots, \iteridx\}$. We do not have access to the full $p_\utility^*(\location) = \frac{\af(\location)\prior(\location)}{\int_\domain \af(\location')\prior(\location')\diff\basemeasure(\location')}$ due to the intractability of the normalization factor%
\footnote{Recall that $\basemeasure$ represents a base measure of the domain with respect to which the proposal densities are defined, i.e., the counting measure for a discrete domain (our experiments) or the Lebesgue measure for a continuous domain.}
$\int_\domain \af(\location')\prior(\location')\diff\basemeasure(\location')$ nor the full acquisition function $\af(\location) = \expectation[\utility(\observation)|\location]$, as we only observe noisy utilities $\utility(\observation_i)$. The normalization factor, however, is constant, and the acquisition function can be unbiasedly approximated via Monte Carlo. Using single-sample estimates for the latter, we then obtain our final form:
\begin{equation}
    \expectation_{\location\sim p_\utility^*}[
            \log \proposal(\location)
        ]
    \approx \frac{1}{\iteridx\nbatch}\sum_{j=1}^\iteridx \sum_{i=1}^\nbatch 
    \frac{\prior(\location_{j,i})\utility(\observation_{j,i})}{\proposal_{j-1}(\location_{j,i})}
            \log \proposal(\location_{j,i})
    = \frac{1}{\nobs_\iteridx}
    \sum_{i=1}^{\nobs_\iteridx}
    \frac{\prior(\location_i)\utility(\observation_{i})}{\proposal_{i-1}(\location_{i})}
            \log \proposal(\location_{i}),
\end{equation}
where the latter follows by simple re-indexing, with $\nobs_\iteridx := \iteridx\nbatch$. We can also drop the constant $\frac{1}{\nobs_\iteridx}$ during optimization.

\paragraph{Balanced forward KL.} The balanced forward KL arises from the definition of Bregman divergences:
\begin{equation}
    \divergence{p}{q} = \Psi(p) - \Psi(q) - \inner{\nabla \Psi(q), p - q}, \qquad p, q \in \Omega,
\end{equation}
where $\Psi: \Omega \to \R$ is a convex function over a convex subset $\Omega$ of a vector space. In our case, $\Omega$ is the space $\Pspace(\domain)$ of probability measures over the domain $\domain$, which can be embedded as convex subset of $\lpspace{2}(\basemeasure)$ when restricted to measures that have a square-integrable density with respect to the base measure $\basemeasure$. Now consider the negative entropy functional:
\begin{equation}
    \Psi(p) = \int_\domain p(\location) \log p(\location)\diff\basemeasure(\location).
\end{equation}
Its functional gradient is given by:
\begin{equation}
    \nabla\Psi(p) = 1 + \log p.
\end{equation}
Thus, the Bregman divergence with this functional is:
\begin{equation}
    \begin{split}
        \divergence{p}{q}
        &= \Psi(p) - \Psi(q) - \inner{\nabla \Psi(q), p - q}\\
        &= \int_\domain p(\location) \log p(\location)\diff\basemeasure(\location)
        - \int_\domain q(\location) \log q(\location)\diff\basemeasure(\location)
        - \inner{1 + \log q, p - q}_{\lpspace{2}(\basemeasure)}\\
        &= \int p \log p \diff\basemeasure
        - \int q \log q \diff\basemeasure
        - \int p - q \diff\basemeasure - \int (p - q)\log q \diff\basemeasure\\
        &= \int p (\log p - \log q) \diff\basemeasure - \int p \diff\basemeasure + \int q \diff\basemeasure.
    \end{split}
\end{equation}
We could cancel the constants $\int p \diff\basemeasure = \int q \diff\basemeasure = 1$. Instead, we keep the term $\int q \diff\basemeasure$, which after an importance sampling approximation, allows us to have a non-banishing term for observations where $\utility(\observation) = 0$. Namely, using a proposal $\proposal_0$, we have that:
\begin{equation}
    \begin{split}
        \divergence{p_\utility^*}{\proposal}
        % &= \expectation_{\location \sim p_\utility^*}[\log p_\utility^*(\location) - \log \proposal(\location)] 
        &= \int \proposal_0 \frac{p_\utility^*}{\proposal_0} (\log p_\utility^* - \log \proposal) \diff\basemeasure + \int \proposal_0 \frac{\proposal}{\proposal_0} \diff\basemeasure - \int p_\utility^* \diff\basemeasure\\
        &= -\int \proposal_0 \frac{p_\utility^*}{\proposal_0} \log \proposal \diff\basemeasure + \int \proposal_0 \frac{\proposal}{\proposal_0} \diff\basemeasure + \anyconstant,
    \end{split}
\end{equation}
where $\anyconstant := \int \proposal_0 \frac{p_\utility^*}{\proposal_0} \log p_\utility^* \diff\basemeasure - \int p_\utility^*\diff\basemeasure$ is constant.
Dropping constants and following a similar approach to the derivation of the forward KL loss, we get:
\begin{equation}
    \begin{split}
        -\int p_\utility^* \log \proposal \diff\basemeasure + \int \proposal \diff\basemeasure
        &\approx -\sum_{i=1}^{\nobs_\iteridx} \frac{\prior(\location_i)\utility(\observation_i)}{\proposal_{i-1}(\location_i)} \log \proposal(\location_i) - \frac{\proposal(\location_i)}{\proposal_{i-1}(\location_i)}, \qquad \iteridx\geq 1.
    \end{split}
\end{equation}

\section{Additional Experimental Detail}
\label{app:experiment-settings}
This section presents algorithmic settings and implementation details for our experiments. We also present a summary of our experimental results in \autoref{sec:summary} and ablation studies in \autoref{sec:ablation}. In addition to the descriptions in this section, code for our experiments is available online at \url{https://github.com/csiro-funml/generativebo}.

\subsection{Text optimization}\label{asub:textgen}

We use the same annealing threshold scheme for setting $\thresh_t$ as \citet[Eqn.\ 20]{Steinberg2025variational}, where we set $\eta$ such that we begin at $p_0=0.5$ we end at $p_T=0.99$.
For the proposal distribution, we found these short sequences best generated by the simple mean-field categorical model,
\begin{align}
    \proposal(\location | \phi) = \prod^M_{m=1} \mathrm{Categ}(x_m | \mathrm{softmax}(\phi_m))
     \label{eq:mfproposal}
\end{align}
where $x_m \in \vocab$ and $\phi_m \in \R^{|\vocab|}$, and we directly optimize $\phi$. VSD and CbAS use the simple MLP classifier guide in \autoref{fig:cpe-arch}.

\begin{figure}[htb]
    \centering
\begin{minipage}{.45\textwidth}%
\begin{lstlisting}
Sequential(
    Embedding(
        num_embeddings=A,
        embedding_dim=16
    ),
    Dropout(p=0.1),
    Flatten(),
    LeakyReLU(),
    Linear(
        in_features=16 * M,
        out_features=64
    ),
    LeakyReLU(),
    Linear(
        in_features=64,
        out_features=1
    ),
)
\end{lstlisting}
\centering
\vspace{6.35cm}
(a) MLP architecture
\end{minipage}%
%\hfill
\begin{minipage}{.45\textwidth}%
\begin{lstlisting}
Sequential(
    Embedding(
        num_embeddings=A,
        embedding_dim=E
    ),
    Dropout(p=0.2),
    Conv1d(
        in_channels=E,
        out_channels=C,
        kernel_size=Kc,
    ),
    LeakyReLU(),
    MaxPool1d(
        kernel_size=Kx,
        stride=Sx,
    ),
    Conv1d(
        in_channels=C,
        out_channels=C,
        kernel_size=Kc,
    ),
    LeakyReLU(),
    MaxPool1d(
        kernel_size=Kx,
        stride=Sx,
    ),
    Flatten(),
    LazyLinear(
        out_features=H
    ),
    LeakyReLU(),
    Linear(
        in_features=H,
        out_features=1
    ),
)
\end{lstlisting}
\centering
(b) CNN architecture
\end{minipage}%
\caption{
    Classifier architectures used for VSD and CbAS in the experiments using PyTorch syntax. $\texttt{A} = \card{\vocab}$, $\texttt{M} = M$, and we give all other parameters in \autoref{tab:cnnsettings} if not directly indicated.}
    \label{fig:cpe-arch}
\end{figure}

\subsection{Protein design}\label{asub:protdesign}

We use the same threshold function and setting for all of the protein design experiments as in \autoref{asub:textgen}. However, these tasks require a more sophisticated generative model that can capture local and global relationships that relate to a protein's 3D structure. For this we use the auto-regressive (causal) transformer architecture also used in \citet{Steinberg2025variational},
\begin{align}
    \proposal(\location | \phi) &= \mathrm{Categ}(x_1 | \mathrm{softmax}(\phi_1))
    \prod^M_{m=2} \proposal(x_m|x_{1:m-1}, \phi_{1:m}), \quad \textrm{where} \nonumber \\
    \proposal(x_m|x_{1:m-1}, \phi_{1:m}) &=
    \mathrm{Categ}(x_m | \mathrm{DTransformer}(x_{1:m-1}, \phi_{1:m})).
\end{align}
For the latter, see Algorithm 10 and 14 in \citet{phuong2022formal} for maximum likelihood training and sampling implementation details, respectively. We give the architectural configuration for the transformers in each task in \autoref{tab:qsettings}, and the classifier CNN used by VSD and CbAS is in \autoref{fig:cpe-arch}.

\begin{table}[htb]
    \centering
    \begin{tabular}{r|c c c c c c}
        Configuration & Stability & SASA & Ehrlich 15 & Ehrlich 32 & Ehrlich 64 \\
        \hline
        Layers & 2 & 2 & 2 & 2 & 2 \\
        Feedforward network & 256 & 256 & 32 & 64 & 128 \\
        Attention heads & 4 & 4 & 1 & 2 & 3 \\
        Embedding size & 64 & 64 & 10 & 20 & 30
    \end{tabular}
    \caption{Transformer backbone configuration.}
    \label{tab:qsettings}
\end{table}

\begin{table}[htb]
    \centering
    \begin{tabular}{r|c c c c c c}
        Configuration & Stability & SASA & Ehrlich 15 & Ehrlich 32 & Ehrlich 64 \\
        \hline
        \texttt{E} & 16 & 16 & 10 & 10 & 10 \\
        \texttt{C} & 96 & 96 & 16 & 16 & 16 \\
        \texttt{Kc} & 7 & 7 & 4 & 7 & 7 \\
        \texttt{Kx} & 5 & 5 & 2 & 2 & 2 \\
        \texttt{Sx} & 4 & 4 & 2 & 2 & 2 \\
        \texttt{H} & 192 & 192 & 128 & 128 & 128
    \end{tabular}
    \caption{CNN guide configuration for VSD and CbAS}
    \label{tab:cnnsettings}
\end{table}

We use the following Ehrlich function configurations:
\begin{description}
    \item[$M=15$:] motif length = 4, no.\ motifs = 2, quantization = 4
    \item[$M=32$:] motif length = 4, no.\ motifs = 2, quantization = 4 
    \item[$M=64$:] motif length = 4, no.\ motifs = 8, quantization = 4 
\end{description}

\subsection{GenBO settings}

\begin{table}[htb]
    \centering
    \begin{tabular}{l|l}
        Acronym & Meaning \\
        \hline
        EI & Expected Improvement\\
        PI & Probability of Improvement\\
        sEI & Soft Expected Improvement, i.e., $\mathrm{softplus}(\observation-\thresh)$\\
        % SR & Simple Regret (utility function)\\
        fKL & Forward KL loss\\
        bfKL & Balanced forward KL loss\\
        rPL & Robust preference loss\\
        MF & Mean-field categorical proposal model\\
        Tfm & Transformer proposal model\\
        fr & More frequent regularization (change in $\regfactor_\nobs$ schedule rate) \\
        r0p10 & Base regularization factor set to $\regfactor_0 := 0.1$ \\
        exp & Exponential regularizer, i.e., $\regfun_\nobs(\parameters) := \regfactor_\nobs\exp \norm{\parameters - \parameters_0}_2^2$\\
        np & No (informative) prior, i.e., $\prior(\location) \propto 1$\\
        p & Pre-trained prior, learned from initial (randomly initialized) data $\dataset_0$\\
        lg & Importance weights\\
        lr0p10 & Learning rate setting for training the generative model (e.g., 0.1 in this case)\\
        pcmin0p50 & Minimum percentile for threshold $\thresh_\iteridx$ annealing schedule (e.g., 50\% in this case)\\
        pcmax0p90 & Maximum percentile for threshold $\thresh_\iteridx$ annealing schedule (e.g., 90\% in this case)\\
    \end{tabular}
    \caption{GenBO experiment settings acronyms}
    \label{tab:genbo-settings}
\end{table}

\autoref{tab:genbo-settings} presents our settings for the different GenBO variants across experiments. The settings for our proposal models followed VSD's configurations. Our regularization scheme penalized the Euclidean distance between the model's parameters and their random initialization \citep{He2015nninit} with $\regfun_\nobs(\parameters) := \regfactor_\nobs\norm{\parameters - \parameters_0}_2^2$, using an annealed regularization factor $\regfactor_\nobs := \regfactor_0 \log^2 \nobs$, similar to \citet{Dai2022}, which ensures enough exploration, while still $\frac{\regfactor_\nobs}{\nobs} \to 0$ as $\nobs\to\infty$, allowing for convergence to the optimal $\parameters_*$. For threshold-based utilities, we mainly set the quantile threshold $\thresh_\iteridx$ to follow an annealing schedule ranging from the 50\textsuperscript{th} (i.e., the median) to the 99\textsuperscript{th} percentile of the observations marginal distribution for both GenBO and VSD, where the percentile $\gamma_\iteridx$ corresponding the quantile is updated as $\gamma_\iteridx := \gamma_{\iteridx-1}^\eta$, where $\eta := \left(\frac{\log \gamma_\niter}{\log \gamma_0}\right)^\frac{1}{\niter-1} \in (0,1)$.

\subsection{Results summary}
\label{sec:summary}
Besides the plots in \autoref{sec:experiments}, we summarize the final results in \autoref{tab:results} and \ref{tab:results-foldx}.

\begin{table}[t]
\centering
\begin{tabular}{lllll}
\toprule
 & ALOHA & Ehrlich-15 & Ehrlich-32 & Ehrlich-64 \\
\midrule
Random mut. & 3.80 $\pm$ 0.40 &  &  &  \\
LaMBO-2 &  & 0.19 $\pm$ 0.17 & 0.36 $\pm$ 0.15 & 0.95 $\pm$ 0.02 \\
CbAS & 2.20 $\pm$ 0.40 & 0.57 $\pm$ 0.12 & 0.61 $\pm$ 0.10 & 0.98 $\pm$ 0.01 \\
GA &  & 0.45 $\pm$ 0.12 & 0.61 $\pm$ 0.10 & 0.98 $\pm$ 0.01 \\
VSD  & \textbf{0.00 $\pm$ 0.00} &  0.19 $\pm$ 0.17  &  0.32 $\pm$ 0.09  &  0.97 $\pm$ 0.00 \\
GenBO  &  0.20 $\pm$ 0.40  & \textbf{0.00 $\pm$ 0.00} & \textbf{0.28 $\pm$ 0.16} & \textbf{0.94 $\pm$ 0.02} \\
\bottomrule
\end{tabular}
\caption{Final average regret (lower is better) for the best-performing variant of each method across the ALOHA (text optimization) and Ehrlich benchmarks}
\label{tab:results}
\end{table}

\begin{table}[t]
\centering
\begin{tabular}{lll}
\toprule
 & FoldX (Stability) & FoldX (SASA) \\
\midrule
Random mut. & 2.79 $\pm$ 0.22 & 12550.26 $\pm$ 56.34 \\
LaMBO-2 & 3.19 $\pm$ 0.58 & 12456.10 $\pm$ 126.64 \\
CbAS & 3.65 $\pm$ 0.23 & 12376.65 $\pm$ 298.30 \\
VSD  & \textbf{4.20 $\pm$ 0.42} &  12537.97 $\pm$ 186.35 \\
GenBO  &  3.28 $\pm$ 0.35  & \textbf{13285.42 $\pm$ 221.60} \\
\bottomrule
\end{tabular}
\caption{FoldX average maximum outcome for the best-performing variant of each method}
\label{tab:results-foldx}
\end{table}

\begin{table}[t]
\centering
\begin{tabular}{ll}
\toprule
Method & Average Runtime \\
\midrule
CbAS & 53.38 s $\pm$ 2.05 s \\
VSD & 42.89 s $\pm$ 2.56 s \\
GenBO & \textbf{14.88 s $\pm$ 0.26 s} \\
\bottomrule
\end{tabular}
\caption{Average run times with standard deviations for the \texttt{ALOHA} text optimization problem.}
\label{tab:runtime}
\end{table}

\subsection{Ablations}
\label{sec:ablation}
We performed ablation studies on the annealing scheme and the evaluations batch size. We vary the minimum and maximum percentile for the threshold annealing settings of both GenBO (with PI utility) and VSD on the text optimization problem in \autoref{fig:annealing}. Plots reveal a preference towards a more exploitative behavior for this simpler optimization problem. In \autoref{fig:batchsize-ablation}, we vary the evaluations batch size $\nbatch$ for GenBO on the Ehrlich benchmark problem of sequence length 32. As expected, larger evaluation batches lead to lower regret, though with higher variability across the candidates.

\begin{figure}[t]
    \centering
    \subcaptionbox{GenBO}{\includegraphics[width=0.475\linewidth]{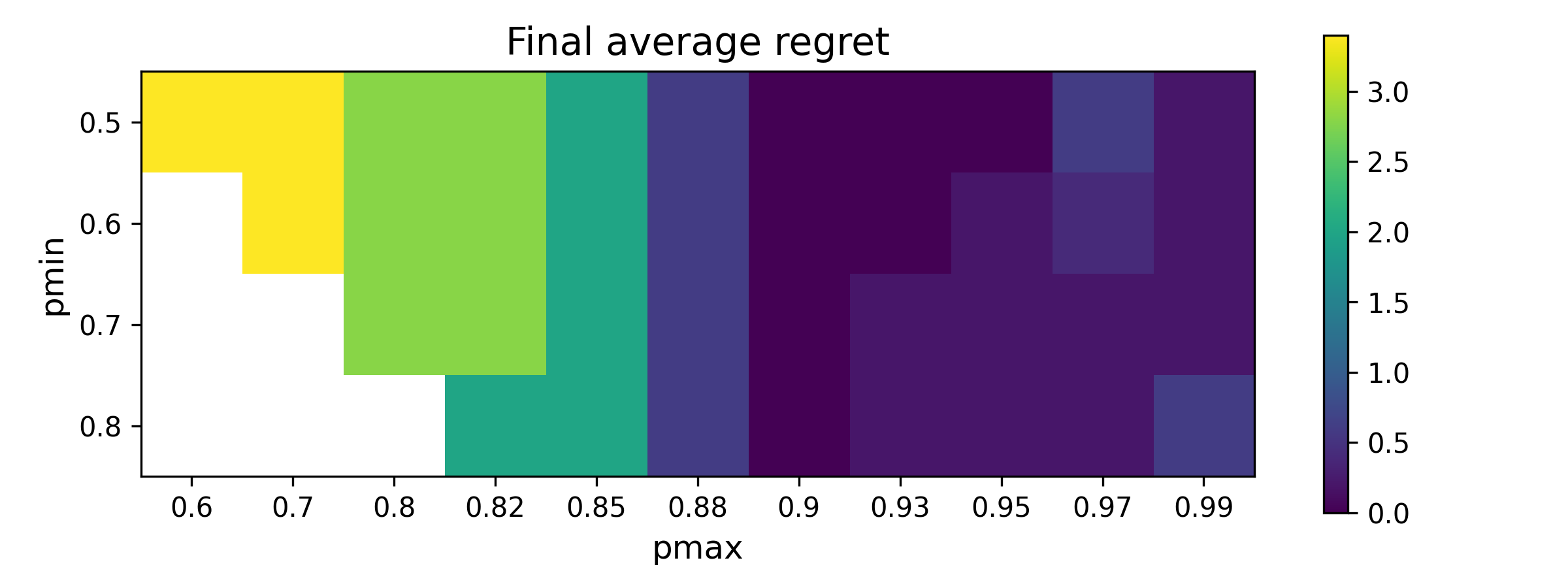}}\subcaptionbox{VSD}{\includegraphics[width=0.45\linewidth]{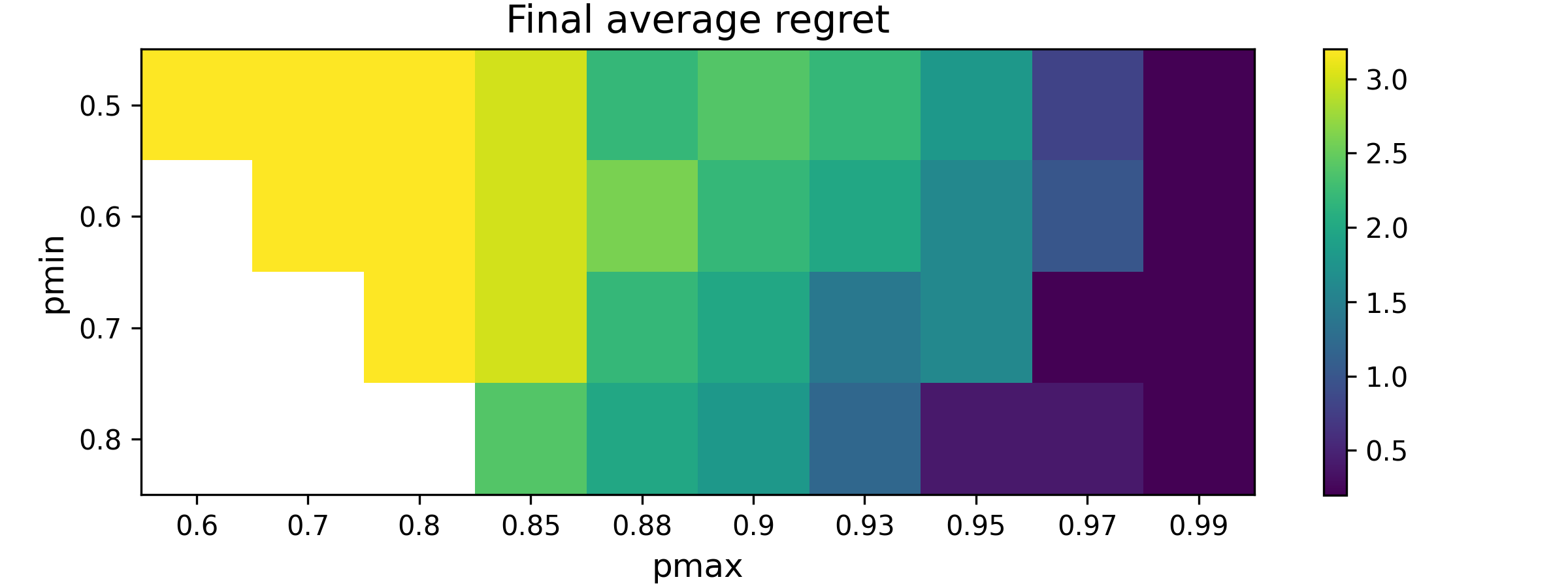}}
    \caption{Final average simple regret (the lower the better) for GenBO and VSD as a function of the minimum and maximum percentile in the annealing schedule for the text optimization problem.}
    \label{fig:annealing}
\end{figure}

\begin{figure}[t]
    \centering
    \includegraphics[width=0.35\linewidth]{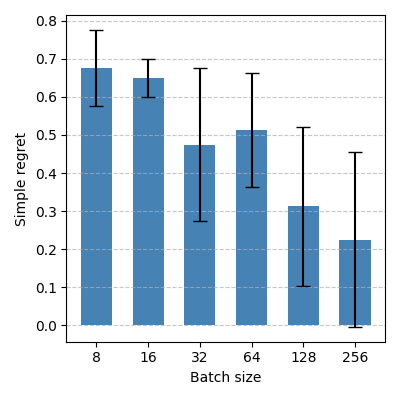}
    \caption{Batch evaluation size $\nbatch$ ablation on Ehrlich benchmark of length 32. The plot presents the final average simple regret for each $\nbatch$ setting, with error bars corresponding to $\pm 1$ standard deviation. All variants were run for the same number of BO rounds as in the original experiment.}
    \label{fig:batchsize-ablation}
\end{figure}

\end{document}

%% file: notation.tex
\newcommand*{\argmax}{\operatornamewithlimits{argmax}}
\newcommand*{\argmin}{\operatornamewithlimits{argmin}}

\newcommand*{\field}[1]{\mathbb{\MakeUppercase{#1}}}		% scalar field
\newcommand*{\set}[1]{{\mathcal{\MakeUppercase{#1}}}}			% set symbol
\newcommand*{\norm}[1]{\lVert #1 \rVert}	% norm of a vector
\newcommand*{\flexnorm}[1]{\left\lVert #1 \right\rVert}
\newcommand*{\inner}[1]{\langle #1 \rangle}	% inner product between vectors, e.g. \inner{v,w}
\newcommand*{\flexinner}[1]{\left\langle #1 \right\rangle}	% auto-scaling version of \inner{}
\newcommand*{\collection}[1]{{\mathfrak{\MakeUppercase{#1}}}} % collection of sets
\newcommand*{\card}[1]{\lvert #1 \rvert}

\newcommand*{\R}{\field{R}} % field of real numbers, or the real line
\newcommand*{\N}{\field{N}} % the set of natural numbers, not a field
\newcommand*{\sign}{\operatorname{sign}}

\renewcommand{\vec}[1]{{\boldsymbol{\mathbf{#1}}}}
\newcommand*{\mat}[1]{\vec{\MakeUppercase{#1}}}
\newcommand*{\eye}{\mat{I}}							% the identity matrix
\newcommand*{\transpose}{\mathsf{T}}
\newcommand*{\idop}{\operator{I}}

\newcommand*{\lpspace}[1]{\set{L}_{#1}}

\newcommand*{\operator}[1]{{{\MakeUppercase{#1}}}}

\newcommand*{\dataset}{\set{D}} 				% observations dataset
\newcommand*{\observation}{y} 					% noisy output observation
\newcommand*{\observations}{\vec{\observation}}
\newcommand*{\obsnoise}{\epsilon}

\newcommand*{\gp}{\mathcal{GP}}
\newcommand*{\gpmean}{\hat{\objective}}

\newcommand*{\kernel}{k}

\newcommand*{\Kernel}{\mat{\kernel}}

\newcommand*{\expectation}{\mathbb{E}}
\newcommand*{\Ex}[2][]{\expectation_{#1}\left[ #2 \right]}

\newcommand*{\Pspace}{\set{P}} 					% set of probability measures

\newcommand*{\prob}[1]{\mathbb{P}\left[ #1 \right]}

\newcommand*{\D}{\mathbb{D}}
\newcommand*{\divergence}[2]{\D(#1||#2)}

\newcommand*{\kl}[2]{\D_\mathrm{KL}(#1||#2)}
\newcommand*{\entropy}[1]{\mathbb{H}[#1]}
\newcommand*{\normal}{\mathcal{N}}					% normal distribution

\newcommand*{\indicator}{\mathbb{I}} 				% indicator function
\newcommand*{\diff}{{\mathop{}\operatorname{d}}}

\newcommand*{\filtration}{\collection{F}}

\newcommand*{\qfamily}{\set{Q}}
\newcommand*{\prior}{p_0}

\newcommand*{\Hspace}{\set{H}} 						% Hilbert space
\newcommand*{\feature}{\phi}
\newcommand*{\features}{\operator{\Phi}}

\newcommand*{\Hessian}{\operator{H}}

\newcommand*{\af}{a} 						% acquisition function
\newcommand*{\Sspace}{\set{S}} 				% search space, in general
\newcommand*{\regret}{r}					% instant regret
\newcommand*{\objective}{f}

\newcommand*{\slocation}{x} % single query coordinate value
\newcommand*{\location}{\vec{\slocation}} % query, vector of coordinates
\newcommand*{\domain}{\set{X}} % query domain, decision set, search space
\newcommand*{\nbatch}{B}
\newcommand*{\batch}{\set{B}}

\newcommand*{\vocab}{\set{V}}

\newcommand*{\model}{g}

\newcommand*{\parameter}{\theta}
\newcommand*{\parameters}{{\parameter}}
\newcommand*{\paramdim}{M}
\newcommand*{\paramspace}{\Theta} % set of parameters, each for a single input location

\newcommand*{\weight}{w}    % to define arbitrary vectors formed by linear combinations
\newcommand*{\loss}{\ell}
\newcommand*{\Loss}{L}

\newcommand*{\lossfactor}{\alpha_\loss}
\newcommand*{\smoothness}{\eta}

\newcommand*{\gap}{\Delta}
\newcommand*{\observer}{m}
\newcommand*{\observers}{\operator{M}}
\newcommand*{\observerset}{\set{M}}

\newcommand*{\basemeasure}{\mu}
\newcommand*{\score}{S}
\newcommand*{\mlabel}{z}                        % model observation/label
\newcommand*{\vnoise}{\xi}
\newcommand*{\vncov}{\Sigma}

\newcommand*{\unitf}{\iota}

\newcommand*{\pflip}{p_{\mathrm{flip}}}
\newcommand*{\PL}{{\mathrm{PL}}}
\newcommand*{\rPL}{{\mathrm{rPL}}}

\newcommand*{\regfactor}{\lambda}
\newcommand*{\regfun}{R}

\newcommand{\utility}{u}

\newcommand{\udataset}{\dataset^{\utility}}

\newcommand*{\clabel}{z}    % classification label
\newcommand*{\thresh}{\tau}    % threshold
\newcommand*{\classifier}{\pi}
\newcommand*{\context}{c}

\newcommand*{\proposal}{q}

\newcommand*{\reward}{\rho}
\newcommand*{\refmodel}{\proposal_{\mathrm{ref}}}
\newcommand*{\normfactor}{\zeta}
\newcommand*{\temperature}{\eta}

\newcommand*{\iteridx}{t}

\newcommand*{\niter}{T}
\newcommand*{\nobs}{n}

\newcommand*{\anotherspace}{\set{Y}}

\newcommand*{\bigo}{\set{O}}        % big-O notation for convergence analysis

\newcommand*{\bound}{b}		% RKHS norm bound for theoretical results

\newcommand*{\anyconstant}{c}

\newcommand*{\anyscalar}{s}

\newcommand*{\anyfunction}{h}

\def\mmiddle#1{\mathrel{}\middle#1\mathrel{}} % Automatic-spacing middle command (originally from: https://tex.stackexchange.com/a/431820)

\newcommand*{\iid}{i.i.d.\xspace}

%% file: theorems.tex
\declaretheorem{assumption}

\declaretheorem[style=remark]{remark}
\declaretheorem{lemma}
\declaretheorem{definition}